\documentclass[11pt,letterpaper]{article}
\usepackage{fullpage}

\usepackage{latexsym,graphicx}
\usepackage{amsmath,amssymb,enumerate}
\usepackage{amsthm}
\usepackage{thmtools, thm-restate}

\newtheorem{theorem}{Theorem}[section]
\newtheorem*{theorem*}{Theorem}

\newtheorem*{proposition*}{Proposition}
\newtheorem{lemma}[theorem]{Lemma}
\newtheorem*{lemma*}{Lemma}
\newtheorem{corollary}[theorem]{Corollary}
\newtheorem*{conjecture*}{Conjecture}
\newtheorem{fact}[theorem]{Fact}
\newtheorem*{fact*}{Fact}

\newtheorem*{hypothesis*}{Hypothesis}

\newtheorem{claim}[theorem]{Claim}
\newtheorem*{claim*}{Claim}

\theoremstyle{definition}

\theoremstyle{remark}

\newtheorem*{remark*}{Remark}

\newtheorem*{observation*}{Observation}

\newcommand{\eat}[1]{}

\newcommand{\Esymb}{\mathbb{E}}
\newcommand{\Psymb}{\mathbb{P}}

\DeclareMathOperator*{\E}{\Esymb}
 
 \DeclareMathOperator*{\ProbOp}{\Psymb}

\renewcommand{\Pr}{\ProbOp}

\renewcommand{\epsilon}{\varepsilon}

\newif\ifnotes\notesfalse

\ifnotes
\usepackage{color}
\definecolor{mygrey}{gray}{0.50}
\newcommand{\notename}[2]{{\textcolor{blue}{\footnotesize{\bf (#1:} {#2}{\bf ) }}}}

\else

\newcommand{\notename}[2]{{}}

\fi
\newcommand{\vnote}[1]{{\notename{Vaidehi}{#1}}}
\newcommand{\anote}[1]{{\notename{Aravindan}{#1}}}

\newcommand{\opt}{\textsc{Opt}}
\newcommand{\alg}{\mathcal{A}}
\newcommand{\minwidth}{\mathsf{minwidth}}
\newcommand{\sequences}{\mathcal{S}}
\newcommand{\unif}{\mathrm{Unif}}
\newcommand{\Icurrent}{I_{\mathrm{current}}}
\newcommand{\tprev}{{t_{\mathrm{prev}}}}
\newcommand{\vol}{\mathrm{vol}}

\newcommand{\distribution}{\mathcal{D}}

\newcommand{\muavg}{\mu_{\mathrm{avg}}}
\newcommand{\mumax}{\mu_{\mathrm{max}}}
\newcommand{\volavg}{\mathrm{vol}_{\mathrm{avg}}}
\newcommand{\volmax}{\mathrm{vol}_{\mathrm{max}}}
\newcommand{\error}{\mathrm{err}}
\usepackage{xspace}
\usepackage{bm}
\usepackage{ifpdf}
\usepackage{shadow,shadethm,color}
\usepackage[compact]{titlesec}
\usepackage{algorithm}
\usepackage[noend]{algorithmic}
\usepackage{verbatim}
\usepackage{paralist}
\usepackage[round]{natbib}

\usepackage{float}
\newfloat{algorithm}{t}{lop}

\usepackage{enumerate}
\usepackage{changepage}

\usepackage[dvipsnames]{xcolor}
\usepackage{tikz}
\usepackage{caption}
\usepackage{subcaption}
\usetikzlibrary{arrows}
\usetikzlibrary{positioning}

\allowdisplaybreaks

\definecolor{Darkblue}{rgb}{0,0,0.4}
\definecolor{Brown}{cmyk}{0,0.81,1.,0.60}
\definecolor{Purple}{cmyk}{0.45,0.86,0,0}
\newcommand{\mydriver}{hypertex}
\ifpdf
 \renewcommand{\mydriver}{pdftex}
\fi
\usepackage[breaklinks,\mydriver]{hyperref}
\hypersetup{colorlinks=true,
            citebordercolor={.6 .6 .6},linkbordercolor={.6 .6 .6},%
citecolor=Darkblue,urlcolor=black,linkcolor=Darkblue,pagecolor=black}
\newcommand{\lref}[2][]{\hyperref[#2]{#1~\ref*{#2}}}
\usepackage[nameinlink]{cleveref}


\usepackage{tikz}

\newcommand{\homework}[4]{
  \noindent
  \begin{center}
  \framebox{
    \vbox{
      \hbox to 5.78in { {\bf 15-315: Introduction to Machine Learning} \hfill #2 }
      \vspace{4mm}
      \hbox to 5.78in { {\Large \hfill Homework #1  \hfill} }
      \vspace{2mm}
      \hbox to 5.78in { {\em Name: #3 \hfill} }
       \vspace{1mm}
      \hbox to 5.78in { {\em Andrew ID: #4 \hfill} }
    }
  }
  \end{center}
  \vspace*{4mm}
}

\newcommand{\email}[1]{\href{mailto:#1}{\texttt{#1}}}

\title{Online Conformal Prediction with Efficiency Guarantees}
\author{Vaidehi Srinivas\thanks{\email{vaidehi@u.northwestern.edu}, Department of Computer Science, Northwestern University, Evanston, USA}}
\date{}

\begin{document}

\maketitle
\thispagestyle{empty} 
\begin{abstract}
    We study the problem of conformal prediction in a novel online framework that directly optimizes efficiency.  In our problem, we are given a target miscoverage rate \(\alpha > 0\), and a time horizon \(T\).  On each day \(t \le T\) an algorithm must output an interval \(I_t \subseteq [0, 1]\), then a point \(y_t \in [0, 1]\) is revealed.  The goal of the algorithm is to achieve coverage, that is, \(y_t \in I_t\) on (close to) a \((1 - \alpha)\)-fraction of days, while maintaining efficiency, that is, minimizing the average volume (length) of the intervals played.  This problem is an online analogue to the problem of constructing efficient confidence intervals.  

We study this problem over arbitrary and exchangeable (random order) input sequences.  For exchangeable sequences, we show that it is possible to construct intervals that achieve coverage \((1 - \alpha) - o(1)\), while having length upper bounded by the best fixed interval that achieves coverage in hindsight. 
For arbitrary sequences however, we show that any algorithm that achieves a \(\mu\)-approximation in average length compared to the best fixed interval achieving coverage in hindsight, must make a multiplicative factor more mistakes than \(\alpha T\), where the multiplicative factor depends on \(\mu\) and the aspect ratio of the problem.
Our main algorithmic result is a matching algorithm that can recover all Pareto-optimal settings of \(\mu\) and number of mistakes.  Furthermore, our algorithm is deterministic and therefore robust to an adaptive adversary.

This gap between the exchangeable and arbitrary settings is in contrast to the classical online learning problem.  In fact, we show that no single algorithm can simultaneously be Pareto-optimal 
for arbitrary sequences and optimal for exchangeable sequences.  On the algorithmic side, we give an algorithm that achieves the near-optimal tradeoff between the two cases.  
\end{abstract}

\newpage
\thispagestyle{empty} 
\tableofcontents

\newpage
\setcounter{page}{1}

\section{Introduction}
Conformal prediction is the problem of generating confidence sets in a \emph{distribution-free} setting.  In the standard setting, there is a ground set of datapoints \(\mathbf{Y} = (Y_1, \dots, Y_{T})\) in some space \(\mathcal{Y}\).  One of these is chosen uniformly at random to be a \emph{test point} \(Y^\star\).  Seeing only the \emph{training data}, i.e., the other data points, \(\mathbf{Y} - Y^\star \), the goal is to construct a confidence set \(C\) for the unseen \(Y^\star\).  For a \emph{miscoverage rate} \(\alpha \in [0, 1]\) that is given as input, a conformal predictor must output a \(C \subseteq \mathcal{Y}\), where \(C = C(\mathbf{Y} - Y^\star)\) is a function of the training data, such that 
\begin{equation}
    \ProbOp_{Y^\star} \left[ Y^\star \in C(\mathbf{Y} - Y^\star) \right] \ge 1 - \alpha. \label{eq:coverage}
\end{equation}
Achieving (\ref{eq:coverage}) is referred to as \emph{coverage}. In addition to coverage, we would like a conformal predictor to be \emph{efficient}, in that it is not too large.  For example, a trivial predictor that always outputs the whole space \(C = \mathcal{Y}\), will achieve coverage, but not necessarily be efficient.  The most commonly studied settings by far are where \(\mathcal{Y}\) is either a discrete label space, or \(\mathcal{Y} = [0, 1]\).
We will focus on the setting where \(\mathcal{Y} = [0, 1]\).  Conformal prediction is a well-studied problem in statistics, and we refer the reader to the book of \citet*{angelopoulos2025theoreticalfoundationsconformalprediction} for a full introduction.

The assumption that the test point is chosen uniformly at random is called \emph{exchangeability}.  This assumption is motivated by the goal of conformal prediction is to produce statistically valid predictions under the weakest possible statistical assumptions.  A typical application of conformal prediction is in uncertainty quantification.  For example, suppose the \(Y_i\) are the errors of a black-box machine learning model on different data points.  Then we can think of the size of a conformal set \(C\) for \(\mathbf{Y}\) as a measure of the uncertainty of the model.  In this setting, we would like to avoid making assumptions about the distribution of model error and data.  For example, we may not be able to assume that the data are independent.  

An important consideration in designing conformal methods is efficiency. 
In the previous example, for instance, to conclude that the conformal method accurately measures the uncertainty of the model, we require a guarantee on efficiency.  Otherwise, if the confidence set is large, we cannot distinguish whether the model is truly uncertain, or whether the conformal method is inefficient on this data.  The conformal prediction literature usually establishes provable guarantees for coverage, and evaluates efficiency empirically.  We will investigate and establish provable efficiency guarantees, in the spirit of a line of work including \citet*{Izbicki2020CDsplitAH, gao2025volumeoptimalityconformalprediction}.

The reason that conformal prediction is built on exchangeability is that it is considered to be the weakest possible assumption that still allows one to make non-trivial predictions that are statistically valid.  However, while exchangeability is much more general than, say, assuming the data are drawn i.i.d., it still rules out many scenarios that arise in practice.  For example, if the data are time-varying, with even a minor distribution shift, the \(Y_i\) are no longer exchangeable, and standard conformal guarantees are no longer valid.  This motivates our main question.
\begin{quote}
    \it \centering
    Is it possible to design a provably efficient conformal predictor for fully arbitrary data?
\end{quote}
It is clear that the type of ``one step" guarantee that is typically studied in conformal prediction is impossible.  If the \(Y_i\)s are fully arbitrary, then a conformal predictor essentially has to output the whole space \(\mathcal{Y}\).\footnote{More formally, the predictor has to cover at least a \((1 - \alpha)\)-fraction of \([0, 1]\) in expectation, which is achieved by a trivial predictor which ignores the training data and outputs a random subset of \([0, 1]\).}
However, in an online setting, we can aim to instead achieve coverage \emph{on average} over time.

\paragraph{Problem statement.} We consider the problem in the following online setting.  Fix a time horizon \(T\).  On each day \(t \le T\), an algorithm \(\alg\) must output a set \(C_t \subseteq [0, 1]\).  Then, \(y_t \in \mathcal{Y}\) is revealed to \(\alg\).  We will refer to the input sequence as \(S = (y_1, \dots, y_T)\). The goal is to achieve coverage close to \(1 - \alpha\), where we now define the coverage of \(\alg\) on \(S\) in an average sense as  
\begin{equation} 
\mathrm{coverage}_{\alg}(S) = \frac{1}{T} \sum_{i = 1}^T \mathbf{1} \left[ y_t \in C_t \right], \label{eq:average-coverage}
\end{equation}
while minimizing the average volume (or Lebesgue measure) of the intervals played
\begin{equation*}
    \mathrm{volume}_{\alg}(S) = \frac{1}{T} \sum_{i = 1}^T \mathrm{volume}(C_t). 
\end{equation*}
We will also refer to the number of points that \(\alg\) does not cover as the number of \emph{mistakes} that \(\alg\) makes (where we want this to be close to \(\alpha T\)). We will adopt the framework of competitive analysis, and aim to design an algorithm that, subject to achieving coverage close to \(1 - \alpha\), achieves average volume that is close to the best fixed interval in hindsight.  That is, we define 
\begin{equation*}
    \opt_{S} (\alpha) = \min_{\text{intervals } I} \mathrm{volume}(I) \quad \text{ s.t. } \frac{1}{T} \sum_{i = 1}^T \mathbf{1} \left[ y_t \in I \right] \ge 1 - \alpha,
\end{equation*}
and our goal is to design an algorithm with low \emph{volume approximation} (competitive ratio) \(\mu\), where we define 
\begin{equation*}
    \mu \ge \frac{\mathrm{volume}_{\alg}(S)}{\opt_{S} (\alpha)}, \quad \text{for all sequences } S.
\end{equation*}
In this paper we will refer to the typical ``one step" problem as \emph{standard conformal prediction}, and this online learning framework as \emph{online conformal prediction}. We note that there are other works that consider conformal prediction in an online setting, notably the line of work started by \citet*{gibbs2025onlineconformal}.  However, they do not provide provable efficiency guarantees, and in fact we show that the strategies developed in previous work \emph{cannot} achieve an efficiency guarantee in the sense of a small volume approximation factor \(\mu\).  We provide a more in-depth comparison in the related work (\Cref{sec:related-work}).  

\subsection{Results}

We will study the problem for \emph{arbitrary order} sequences, and \emph{exchangeable sequences}.  Exchangeable sequences are ones where the sequence \(\mathbf{S} = (\mathbf{S}_1, \dots, \mathbf{S}_{T}) \) is a random variable that satisfies 
\begin{equation*}
    (\mathbf{S}_1, \dots, \mathbf{S}_{T}) \overset{\mathrm{d}}{=} (\mathbf{S}_{\sigma(1)}, \dots, \mathbf{S}_{\sigma(T)}), 
\end{equation*}
for all permutations \(\sigma : [T] \rightarrow [T]\), where \(\overset{\mathrm{d}}{=}\) means they are equal in distribution.  This is also referred to as the \emph{random order} model.  Note that the class of exchangeable sequences generalizes the class of i.i.d.\ sequences, as i.i.d.\ sequences are always exchangeable, but an exchangeable sequence need not be i.i.d..

We observe that this problem can be seen as a \emph{constrained optimization} analogue of the classical \emph{online learning} problem (see \Cref{sec:technical-overview} for discussion).  A beautiful property of online learning is that going from i.i.d.\ sequences to arbitrary sequences is essentially without loss-- the same optimal bounds can be recovered in both cases.  This motivates the following question for our setting.

\begin{quote}
    \it
    Is it possible to design an algorithm for fully arbitrary data that recovers the best possible bounds for exchangeable data?
\end{quote}

Surprisingly, the conformal prediction problem exhibits very different trade-offs for these two settings.  We establish near-optimal bounds for both the arbitrary and exchangeable settings.  We also present a \emph{``no best-of-both worlds"} result, that shows that we cannot hope to design one single algorithm that achieves the best possible guarantee in both the arbitrary and exchangeable settings.  

\paragraph{Results for arbitrary sequences.}  Our first result is an algorithm for arbitrary sequences, which achieves a given volume approximation \(\mu\) while making a bounded number of mistakes.  

\begin{theorem}[Informal version of \Cref{cor:optimal-algorithm-for-arbitrary-order}]
    For a given scale lower bound \(\minwidth > 0\), multiplicative volume approximation \(\mu > 3\), target miscoverage rate \(\alpha \ge 0\), and time horizon \(T\), we give an algorithm that on any sequence \(S\) of length \(T\) plays intervals of maximum volume (and therefore average volume)
    \[\le \mu \max \{ \opt_S(\alpha), \minwidth \},\]
    and makes number of mistakes bounded by 
    \[O \left(\frac{\log (1/\minwidth)}{\log (\mu)} (\alpha T + 1) \right).\]
    Moreover, the algorithm is deterministic and therefore robust to an adaptive adversary.
\end{theorem}
This tells us that, for example, it is possible to achieve a constant volume approximation \(\mu = 5\), and have number of mistakes bounded by \(O(\log(1/\minwidth) \alpha T)\). The algorithm is simple, and is in fact an instantiation of a meta-algorithm, \Cref{alg:meta-algorithm}, that we can apply to the exchangeable setting as well.   \Cref{alg:meta-algorithm} maintains an interval \(\Icurrent\) that has achieved error rate at most \(R(t)\) by day \(t\), where for the arbitrary order setting we choose \(R(t) = \alpha \frac{T}{t}.\)  If on any day, \(\Icurrent\) no longer meets this requirement, \Cref{alg:meta-algorithm} resets \(\Icurrent\) to be a \(\mu\)-scaling of the smallest volume interval that does achieve the error requirement.  The analysis for the arbitrary order case shows that, while this is not explicitly enforced by the algorithm, this procedure ends up doing a doubling search for the scale of the optimal interval.  

\addcontentsline{toc}{subsection}{\emph{\Cref{alg:meta-algorithm}:} Meta-algorithm for Online Conformal Prediction}
\begin{algorithm}
\caption{Meta-algorithm for Online Conformal Prediction}
\label{alg:meta-algorithm}
\begin{algorithmic}[1]
    \REQUIRE{scale lower bound \(\minwidth > 0\), multiplicative volume approximation \(\mu \ge 1\), \\allowable error rate function \(R(t) : \{0, \dots, T - 1\} \rightarrow [0, 1]\)
    }
    \STATE \(\Icurrent \leftarrow [0, 0]\)
    \FOR{day \(t\)}
        \IF{\(\Icurrent\) has empirical coverage \(< 1 - R(t - 1)\) over \(\{y_{t'} : t' < t \}\)} \label{algline:reset-condition}
            \STATE \(\mathcal{F}_t :=\) set of intervals that achieve coverage \(\ge 1 - R(t - 1)\) over \(\{y_{t'} : t' < t \}\)
            \STATE \qquad {\color{gray} \textbf{//} set of \emph{feasible} intervals on day \(t\)}
            \STATE \(I_t := \) arbitrary minimum volume interval \(\in \mathcal{F}_t\)
            \label{algline:feasible-interval}
            \STATE \(\widehat{\mu} := \mu \max\{1 , ~\frac{\minwidth}{\vol(I_t)}\}\) 
            \STATE \qquad {\color{gray} \textbf{//} ensure the interval will have volume \(\ge \minwidth\)}
            \STATE \(\Icurrent \leftarrow \widehat{\mu} I_t\) \label{algline:expand-interval}
            \STATE \qquad {\color{gray} \textbf{//} for \(s > 0\), interval \(I = [a, b]\), define \(s I = \{x \subseteq \mathbb{R} : |x - \frac{a + b}{2}| \le s \cdot \frac{b - a}{2}\}\)}
        \ENDIF 
        \STATE \textbf{play} \(\Icurrent \cap [0, 1]\)
        \STATE \(y_t\) is revealed
    \ENDFOR
\end{algorithmic}
\end{algorithm}

We highlight that the analysis is not specific to the choice of \(R(t)\), and different choices of \(R(t)\) achieve different trade-offs between \(\mu\) and the number of mistakes.  In the introduction, we present the result for the optimal choice of \(R(t)\) for clarity, and we refer the reader to \Cref{thm:meta-alg-arbitrary-order} for full details.  

Our algorithm exhibits tradeoff in the multiplicative volume approximation and the multiplicative approximation in number of mistakes.  We show that this tradeoff is essentially tight. 
\begin{theorem}[Informal version of \Cref{thm:arbitrary-order-lower-bound}]
    For any scale lower bound \(\minwidth > 0\), miscoverage rate \(\alpha > 0\), and time horizon \(T\), there is a set \(\mathcal{S}\) of input sequences of length \(T\) chosen oblivious to any algorithm, such that for any potentially randomized algorithm \(\alg\) that outputs confidence sets over \([0, 1]\) that are not necessarily intervals, 
    \begin{enumerate}
        \item If \(\alg\) plays sets with expected \emph{average} volume 
        \(\le \muavg \max \{ \opt_{S_i}(\alpha), ~ \minwidth \}\)
        on every sequence \(S_i \in \sequences\), for some value \(\muavg>0\), then \(\alg\) must make 
        \[\text{ for any } \varepsilon' > 0, \quad \Omega \left( \min \left\{ \frac{\log(1/\minwidth)}{\log(\muavg)} \alpha^{1 + \varepsilon'} T , ~T \right\} \right) ~\text{mistakes}\]
        in expectation on some sequence \(S_j \in \sequences\). 
        
        \item If \(\alg\) plays sets with expected \emph{maximum} volume 
        \(\le \mumax \max \left\{ \opt_{S_i} (\alpha), ~ \minwidth \right\} \)
        on any sequence \(S_i \in \sequences\), for some value \(\mumax>0\), then \(\alg\) must make 
        \[\Omega \left( \min \left\{ \frac{\log (1/\minwidth)}{\log(\mumax)} \cdot \alpha T , ~T \right\}\right) ~\text{mistakes} \]
        in expectation on some sequence \(S_j \in \sequences\).
    \end{enumerate}
\end{theorem}
Recall that our algorithm is deterministic, robust to an adaptive adversary, proper (it outputs sets that are intervals), and it has a bound on the volume of the largest interval it ever outputs.  This lower bound tells us that it is essentially optimal, even among algorithms that are randomized, only robust to an oblivious adversary, improper (output sets that are not intervals), and have a bound only on the \emph{average} volume of sets that they play.  

\paragraph{Results for exchangeable sequences.}  Since conformal prediction is typically studied in the setting where the data is exchangeable, it is natural to consider this in our online formulation as well.  Note that the guarantees we establish for this setting are incomparable to the standard ``one step" guarantees of conformal prediction: we will establish coverage on average over time (\ref{eq:average-coverage}), as opposed to statistical validity on day \(T\) alone (\ref{eq:coverage}).  

We show that the same meta-algorithm \Cref{alg:meta-algorithm}, instantiated with \(\mu = 1\) and a different choice of \(R(t) = \alpha - O(\sqrt{\log T/T} )\) can achieve an optimal guarantee for exchangeable data. 
\begin{theorem}[Informal version of \Cref{cor:optimal-algorithm-for-exchangeable-sequences}]
    For a given scale lower bound \(\minwidth > 0\), miscoverage rate \(\alpha\), and time horizon \(T\), we have an algorithm that, on an exchangeable sequence \(\mathbf{S}\), with probability \(\ge 1 - \frac{1}{100}\) over the exchangeability of \(\mathbf{S}\), plays intervals of maximum volume 
    \[\le \max \{\opt_{\mathcal{S}}(\alpha), ~\minwidth \},\]
    and achieves expected coverage 
    \(\ge (1 - \alpha) - O(\sqrt{\log T/T}).\)
\end{theorem}
This bound is essentially optimal, as the vanishing error in coverage \(O(\sqrt{\log T/T})\) is what we would get due to sampling error when the sequence is drawn i.i.d., which is a special case of exchangeability.

Similarly to the arbitrary case, our analysis is not specific to the choice of \(R(t)\), and different choices of \(R(t)\) achieve different trade-offs between \(\mu\) and the number of mistakes.  In the introduction, we present the result for the optimal choice of \(R(t)\) for clarity, and we refer the reader to \Cref{thm:meta-alg-exchangeable} for full details.  

As a corollary to this analysis, we show that the interval that \Cref{alg:meta-algorithm} plays on day \(T/2\) is actually a statistically valid and volume optimal conformal set for \(y_T\) (\Cref{cor:efficiency-for-standard-conformal-prediction}).  We note that this guarantee can also be seen as a consequence of the result of \citet*{gao2025volumeoptimalityconformalprediction} along with uniform convergence for exchangeable sequences (\Cref{lem:exchangeable-uniform-convergence}).  We view this observation as evidence that the online conformal prediction problem is more general than the offline version, in a similar spirit to how online learning is a generalization of the offline PAC learning problem.  

\paragraph{``Best of both worlds."}  Since our optimal algorithms for both the arbitrary order setting and the exchangeable setting are instantiations of the same meta-algorithm \Cref{alg:meta-algorithm}, it is natural to ask whether there is a single algorithm that simultaneously achieves optimal bounds in both settings.  That is, is there a single algorithm \(\alg\) such that, on arbitrary inputs \(\alg\) achieves a given Pareto optimal tradeoff in \(\mu\) and number of mistakes, and on exchangeable inputs \(\alg\) achieves the better tradeoff of \(\mu = 1\) and coverage \((1 - \alpha) - o(1)\)?  We show that this is indeed not the case.  

\begin{theorem}[Informal version of \Cref{thm:no-best-of-both-worlds}]
    Fix a scale lower bound \(\minwidth > 0\), target miscoverage rate \(\alpha < \frac{1}{4}\), and multiplicative volume approximation \(\mu > 1\).  Take any (potentially randomized) algorithm \(\alg\) that on any arbitrary order sequence \(S\) plays intervals of expected average volume 
     \[\le \mu \max \{ \opt_{S}(\alpha), ~\minwidth \}.\]
    Then, \(\alg\) must make  
    \[\Omega \left( \min \left\{ \frac{\log(1/\minwidth)}{\log(\mu)},  ~\log(1/\alpha)\right\} \alpha T \right) \]
    mistakes in expectation on some i.i.d.\ sequence \(S'\).
\end{theorem}
That is, an algorithm that achieves a non-trivial guarantee on every arbitrary sequence must incur a multiplicative overhead in the number of mistakes it makes on i.i.d., and therefore exchangeable sequences. 

On the positive side, we show that \Cref{alg:meta-algorithm} with \(R(t) = \frac{\alpha T}{t}\) meets this mistake bound on exchangeable sequences (\Cref{thm:best-of-both-worlds-algorithm}).  Recall that this is the same setting of \(R(t)\) that gives an optimal algorithm for arbitrary sequences.  Thus, this algorithm is optimal for arbitrary sequences, and gets the best possible mistake bound subject to achieving a guarantee for arbitrary sequences.  We remark that \Cref{alg:meta-algorithm} does incur the same multiplicative volume approximation of \(\mu\) for both arbitrary and exchangeable sequences, whereas the lower bound does not rule out the case that a single algorithm can achieve a better volume approximation for exchangeable sequences than arbitrary sequences.

\subsection{Technical Overview}
\label{sec:technical-overview}

\paragraph{Contrast to online learning.} 
This problem is reminiscent of classical online learning.
Think of each interval \(I\) as an ``expert."  In fact, there are infinitely many intervals, and this poses a genuine issue, but for now, assume there \(n < \infty\) experts, and we will justify this assumption later.  On day \(t\), the expert corresponding to interval \(I\) has loss 0 if \(y_t \in I\) and loss 1 if \(y_t \notin I\).  That is, the loss of the expert is exactly the coverage of the corresponding interval.  The classical online learning problem is to pick an expert on each day to achieve average loss that competes with the loss of the best fixed expert in hindsight.  For our particular setting, this is trivial, as there is an expert in our set that corresponds to an interval that covers the whole space, and achieves loss 0 on every day, but the online learning problem is of course much more general.  

In our problem, instead of loss being an objective that we are trying to minimize, loss becomes a constraint: we want to achieve average loss \(\le \alpha\) over the sequence.  The objective that we are trying to minimize is an average over some other function over the experts that we choose, for us volume of the corresponding interval.  In this sense, our problem becomes an online \emph{constrained optimization} problem, where the set of experts that are feasible with respect to the constraint in hindsight depends on the input sequence. 


One of the beautiful properties of classical online learning is that it is possible to design algorithms for arbitrary input sequences that have vanishing regret, i.e., they make \(1 + O(\sqrt{\log n / T})\) times as many errors as the best fixed expert in hindsight (see e.g., \citet*{Blum_Hopcroft_Kannan_2020}).  This vanishing regret term matches the sampling error that one would face even for i.i.d.\ sequences.  In essence, this says that going from i.i.d.\ (and therefore exchangeable) data to arbitrary data is without loss, as we are able to achieve the best possible guarantee even in the arbitrary case. 
Hence there is one algorithm that is able to achieve the ``best of both worlds."  Such results are also known in bandit problems \citep*{bubeck12b}, online conformal prediction (without efficiency guarantees) \citep*{ABB-best-of-both-worlds}, and other settings.  


It turns out that, in contrast to classical online learning, our problem exhibits a large gap between what is possible in the arbitrary setting and the i.i.d.\ setting.  
For exchangeable data, we show that it is possible to design an algorithm that has volume approximation \(\mu = 1\), while having a vanishing fraction of additional mistakes, i.e., the algorithm has coverage \((1 - \alpha) - \widetilde{O}(\sqrt{\log n/T})\), where we will make what \(n\) means in this setting precise shortly, and \(\widetilde{O}\) hides subpolynomial factors in \(T\).
In the arbitrary setting however, we show that any algorithm, even one that is randomized, only robust to an oblivious adversary, and improper (outputs sets that are not intervals), faces a stark tradeoff between volume approximation \(\mu\) and how many times more than \(\alpha T\) mistakes it makes.  Additionally, we show that even subject to this lower bound, there can be no algorithm that is simultaneously optimal for the exchangeable setting and Pareto optimal for the arbitrary setting.  This rules out any ``best of both worlds" guarantee.

On the positive side, we design a matching algorithm that is able to achieve any Pareto optimal tradeoff.  This algorithm is deterministic, which is also in contrast to the online learning problem which exhibits a separation between deterministic and randomized algorithms. 
Furthermore, our algorithms for exchangeable and arbitrary sequences fit in a unified framework that allows us to interpolate between being optimal for one and optimal for the other. 

\paragraph{Challenges for arbitrary sequences.}  Designing an algorithm for arbitrary input sequences presents some challenges.  Consider the following example.  Our goal is to design a conformal predictor for miscoverage rate \(\alpha = \frac{1}{2}\).  
There are two sequences.  Let \(\varepsilon > 0\) be some small parameter. \(S^{(1)}\) is drawn uniformly from \([0, \varepsilon]\) for the first half of the sequence, and then expands in scale to be drawn uniformly from \([0, 1]\) for the second half of the sequence.  \(S^{(2)}\) is also drawn uniformly from \([0, \varepsilon]\) for the first half of the sequence, but then shrinks in scale to be drawn uniformly from \([0, \varepsilon^2]\) for the second half of the sequence. That is, 
\begin{align*}
    S^{(1)} = (S^{(1)}_1, \dots, S^{(1)}_T), \quad& S^{(1)}_t \sim \begin{cases} \mathrm{Unif}[0, \varepsilon] &\text{for } t \le T/2 \\ \mathrm{Unif}[0, 1] &\text{for } t > T/2 \end{cases}, \\ 
    S^{(2)} = (S^{(2)}_1, \dots, S^{(2)}_T), \quad& S^{(2)}_t \sim \begin{cases} \mathrm{Unif}[0, \varepsilon] &\text{for } t \le T/2 \\ \mathrm{Unif}[0, \varepsilon^2] &\text{for } t > T/2 \end{cases} .
\end{align*}
Since half of \(S^{(1)}\) is at scale \(\varepsilon\), and half of \(S^{(2)}\) is at scale \(\varepsilon^2\), we have
\begin{equation*}
    \opt_{S^{(1)}}(\alpha) \le \varepsilon, \qquad \opt_{S^{(2)}}(\alpha) \le \varepsilon^2.
\end{equation*}

Any algorithm \(\alg\) that captures any constant fraction \(c > 0\) of points must incur a large multiplicative error in volume.  

Since \(S^{(1)}\) and \(S^{(2)}\) are the same for the first \(T/2\) days, \(\alg\) must capture at least a \(c/2\)-fraction of points in the first half of both sequences, or in the second half of both sequences. 
If \(\alg\) captures a \(c/2\)-fraction of points on the first half of the sequences, then on sequence \(S^{(2)}\), \(\alg\) has large average volume, and thus large volume approximation:
\begin{equation*}
    \mathrm{volume}_{\alg}(S^{(2)}) \ge \frac{c \varepsilon}{4} \quad \Longrightarrow \quad  \mu \ge \frac{\mathrm{volume}_{\alg}(S^{(2)})}{\opt_{S^{(2)}}(\alpha)} \ge \frac{c}{4 \varepsilon}.
\end{equation*}
However, if \(\alg\) captures a \(c/2\)-fraction of points on second half of the sequences, then on sequence \(S^{(1)}\), \(\alg\) has large volume and thus large volume approximation:
\begin{equation*}
    \mathrm{volume}_{\alg}(S^{(1)}) \ge \frac{c}{4} \quad \Longrightarrow \quad  \mu \ge \frac{\mathrm{volume}_{\alg}(S^{(1)})}{\opt_{S^{(1)}}(\alpha)} \ge \frac{c}{4 \varepsilon}.
\end{equation*}

This example highlights two issues.  The first is a recurring issue in online learning.  As \(\varepsilon \rightarrow 0\), our volume approximation becomes unbounded.  This arises because the aspect ratio of our problem is unbounded.  Because the problem is scale-free, a non-trivial algorithm that outputs intervals can be forced to make an unbounded number of mistakes.  We address this by introducing a parameter \(\minwidth > 0\), and require an algorithm to compete with the smallest volume interval that achieves coverage, among intervals that have volume at least \(\minwidth\).  This, along with the fact that \(\mathcal{Y}\) is the bounded interval \([0, 1]\), fixes the aspect ratio of the problem.  In essence, this tells us that we may as well think of a grid of size \(1/\minwidth\) over the interval \([0, 1]\), and restrict the problem to considering intervals that are between two gridpoints.  This reduces the number of intervals we have to consider to \(O(1/\minwidth^2)\), which we should think of as analogous to the number of experts \(n\) in the earlier analogy to online learning.  Similarly to online learning, we think of \(n \approx (1/\minwidth)^2\) as very large, and aim for guarantees that depend on \(\log n \approx \log(1/\minwidth)\).  

However, there is a more significant issue.
Even if \(\minwidth\) is fixed, the above example still forces any algorithm that captures a constant fraction of points to incur a multiplicative error \(\Omega(1/\minwidth)\) in volume, which one should think of as being extremely large.  This issue arises because we require the algorithm to achieve coverage very close to the target, i.e., a non-trivial guarantee for \(\alpha = \frac{1}{2}\) requires the algorithm to make \(< T = 2 \alpha T\) mistakes.  Surprisingly, the problem actually becomes tractable for smaller values of \(\alpha\) if we consider a \emph{biciteria approximation}, that competes with the best interval that makes at most \(\alpha T\) mistakes, while making a multiplicative factor more than \(\alpha T\) mistakes.

Algorithmically, this example illustrates the difference between online learning and our problem. In online learning, on day \(t\) an algorithm can choose an expert that does well locally, that is, does well around time \(t\).  While this may not be the best expert over the whole sequence, the fact that it is doing as well as the best expert is doing right now, means that in every local time window, the algorithm is only performing as well or better than the best expert in hindsight.  However, for our problem, choosing a different expert could incur a vastly different cost than the best expert.  For example, if \(\alg\) captures a constant fraction of days in the first half of the sequence, on \(S_2\) it only has coverage even higher than the best expert in hindsight, so it satisfies the coverage constraint.  In exchange, however, it incurs much higher cost than the best expert.  Thus on some days, an efficient algorithm must choose to abstain, and sacrifice coverage to minimize cost.  

Put differently, suppose that each day's input \(y_t\) is drawn from some distribution \(\distribution_t\) (which are potentially different for each day).  For each day \(t\), define \(I_t^\star\) to be the minimum volume interval that achieves coverage \(1 - \alpha\) over \(\distribution_t\).
We refer to the sequence of intervals \((I^\star_1, \dots, I^\star_T)\) as the \emph{daily optimum}.  It is not necessarily the case that the daily optimum has smaller average volume than the best single interval that achieves average coverage \(1 - \alpha\) over the whole sequence of distributions \(\distribution_t\):  
\begin{align*}
    \frac{1}{T} \sum_{i = 1}^T \mathrm{volume}(I^\star_t) &= \frac{1}{T} \sum_{i = 1}^T \min_{I_t} \mathrm{volume}(I_t) \text{ s.t. } \ProbOp_{y \sim \distribution_t} [y \in I_t] \ge 1 - \alpha\\
    &~\not\le~ \min_I \mathrm{volume}(I) \text{ s.t. } \left(\frac{1}{T} \sum_{i = 1}^T \ProbOp_{y \sim \distribution_t}[y \in I] \right) \ge 1 - \alpha. 
\end{align*}

Thus the coverage constraint is a \emph{global constraint} that requires an approach beyond trying to solve each individual day's problem.  We discuss this further in the related work (\Cref{sec:related-work}).


To approach this problem, we show that \Cref{alg:meta-algorithm} essentially implements a doubling search for the volume of the optimal single interval that achieves global coverage \(1 - \alpha\).  In some sense, we can think of this as implementing a optimization to feasibility reduction, which is a classic idea in optimization.  We would like to minimize volume, subject to a global coverage constraint.  It turns out to be easy to instead maximize global coverage subject to a volume constraint.  The doubling search then essentially performs a reduction from this second ``feasibility" problem to the first ``optimization" problem.  This reduction requires a phases for each step of the doubling search, and therefore incurs some overhead.  This leads to the multiplicative factor in the number of mistakes.  

We are able to show that this overhead is necessary, by constructing a matching lower bound.  As it turns out, the lower bound comes from exactly a ``bad instance" for our algorithm, in which we show that no algorithm can do any better.

\paragraph{Challenges for Exchangeable Sequences.}
Our algorithm for exchangeable sequences is another instantiation of the same meta-algorithm \Cref{alg:meta-algorithm}.  Essentially on each day \(t\), the algorithm plays an interval \(I_t\), which is the minimum volume interval that achieves coverage (up to some slack) over the previous days inputs \(S_1, \dots, S_{t - 1}\).

The analysis of the algorithm for exchangeable sequences is tricky because the interval \(I_t\) that \Cref{alg:meta-algorithm} plays on day \(t\) is dependent on the first \(t - 1\) days of the sequence \(S_1, \dots, S_{t - 1}\).  This is in turn \emph{not} independent of \(S_t\).  Thus the coverage on day \(t\)
\[\ProbOp [S_t \in I_t(S_1, \dots, S_{t - 1})],\]
where the probability is over the exchangeability of \(S\), is challenging to bound.  This is actually the exact setting of standard conformal prediction.  However, existing methods only reason about coverage and do not come with provable volume guarantees, which are necessary for this problem.\footnote{To the best of our knowledge, all known volume optimality guarantees for standard conformal prediction require stronger assumptions than just exchangeability on the data. For example, they typically require the data to be drawn i.i.d..  We refer the reader to \Cref{sec:related-work} for more discussion.}

To understand the challenge, consider the following example.  Let \(S\) be a uniformly random permutation of a multiset that has \((1 - 2\alpha)T\) copies of \(1/2\), \(\alpha T\) copies of \(0\) and \(\alpha T\) copies of \(1\). \(S\) is exchangeable by construction.  On day \(T\), we have seen all but one element, which could have been \(1/2\) with probability \(1 - 2 \alpha\), \(0\) with probability \(\alpha\), or \(1\) with probability \(\alpha\).

\Cref{alg:meta-algorithm}, for \(\mu = 1\), \(R(t) = \alpha - O(\sqrt{\log T/T} )\), will arbitrarily choose a minimum volume interval that achieves coverage \(1 - R(T-1)\) on the first \(T - 1\) points.  
The intervals \([0, 1/2]\) or \([1/2, 1]\) will both meet the coverage requirement and have equal (minimal) volume, so say that \Cref{alg:meta-algorithm} tiebreaks by choosing the interval that achieved higher coverage, and choosing \([0, 1/2]\) if they have equal coverage.  
Then, if \(S_{T} = 0\), \([1/2, 1]\) will achieve higher coverage on the prefix and \(S_{T}\) will not be in the output interval.  
Similarly, if \(S_{T} = 1\), \([0, 1/2]\) will achieve higher coverage on the prefix and \(S_{T}\) will not be in the output interval.  
Thus \(S_{T}\) is covered with probability \(\le 1 - 2 \alpha\), when we are aiming for coverage \(1 - \alpha\).  Our analysis handles this subtle issue by showing that there cannot be many days \(t\) on which this issue arises.  This allows us to amortize the loss in coverage by allowing for higher error rates both at the beginning and at the end of the sequence.  Carefully accounting for this allows us to show that the standard analysis for i.i.d.\ sequences goes through even for sequences that are only exchangeable and not necessarily i.i.d..  

A key lemma that allows us to get these bounds is the fact that \emph{uniform convergence} for set families of bounded VC-dimension, for example intervals over \([0, 1]\), holds over not only i.i.d.\ inputs, but over exchangeable inputs as well (\Cref{lem:exchangeable-uniform-convergence}).  While this is already known (we provide the writeup in \Cref{sec:uniform-convergence} for completeness), the contribution of this work is to use this to amortize error across the whole exchangeable sequence.  


\subsection{Related Work}
\label{sec:related-work}

Conformal prediction is a very well-studied area in statistics with a wealth of recent work.  The theoretical work alone in this space is the topic of a recent textbook \citep*{angelopoulos2025theoreticalfoundationsconformalprediction}.  We provide a brief comparison to some of the work most relevant to the setting we consider, and refer the reader to the textbook for a more in-depth treatment.  

\paragraph{Efficiency guarantees for standard conformal prediction.}  Typically, conformal prediction methods are required to provably achieve coverage (\ref{eq:coverage}).  Efficiency, however, is often empirically validated on real and simulated data.  There is a line of work that tries to establish efficiency guarantees subject to provably achieving coverage.  The work of \citet*{Sadinle02012019}, which considers the problem when the set \(\mathcal{Y}\) is a discrete label space. The line of work including \citet*{Lei01032013, izbicki20a, Izbicki2020CDsplitAH} work in settings where the data is drawn i.i.d. from a distribution that allows for a good p.d.f.\ estimator, and then achieving a volume optimal prediction set corresponds to taking a level set of that estimator.  However, this approach does not work for arbitrary distributions, that may in general be hard to approximate.
The work of \citet*{gao2025volumeoptimalityconformalprediction} shows that when the data is drawn i.i.d., if one chooses a prediction set from a family of potential sets that has bounded VC-dimension, it is sufficient to find the most efficient prediction set that achieves coverage over the samples.  
To the best of our knowledge, all known volume optimality bounds for standard conformal prediction must make stronger assumptions on the data than just exchangeability.      

\paragraph{Beyond exchangeability.}  A line of work has tried to relax the assumption of exchangeability for standard conformal prediction. 
A line of work including \citet*{tibshirani2019, Barber2022ConformalPB} designs conformal methods that work when the joint distribution of the data is ``approximately" exchangeable.  \citet*{prinster2024} design conformal methods when the joint distribution of the data is not exchangeable, but requires that the joint distribution is known in advance. 

\paragraph{Online conformal prediction.}  For exchangeable sequences, a foundational result of \citet*{vovk2002} shows that a standard conformal method when run repeatedly on prefixes of an online sequence, errs with probability that is independent on each day, and is therefore valid on every day.  For arbitrary sequences, the problem of computing confidence sets that achieve coverage on average was initiated by \citet*{gibbs2021}, and studied further by \citet*{zaffran22a, gibbs2025onlineconformal, bhatnagar23}.  
\citet*{feldman2023achieving, angelopolous23} give an algorithm that has coverage provably tending to \(1 - \alpha\) as \(T \rightarrow \infty\).  This model does not explicitly optimize for efficiency.  Instead, the justification is to think of coverage as being a proxy for efficiency: by ensuring that coverage is not too much more than the target \(1 - \alpha\), this incentivizes the algorithm to leave some things out of the prediction set.

To the best of our knowledge, ours is the first work to give provable efficiency guarantees for conformal prediction in the online setting, even for exchangeable data.  In fact, our lower bound (\Cref{thm:arbitrary-order-lower-bound}) proves that any method that achieves coverage tending to \(1 - \alpha\) on arbitrary sequences must incur a very large multiplicative volume approximation in the worst case.  

The algorithm of \citet*{angelopolous23} works by estimating the bottom \(\alpha/2\) and top \(\alpha/2\) quantile of the data in an online fashion.  The confidence interval will then be the interval between these two estimated quantiles.  This is possible since a given quantile is the minimizer of a convex function called the \emph{pinball loss}, so it is possible to get a low regret estimator using \emph{online convex optimization}.  
In our setting of optimizing volume, this runs into a couple of issues.  The first is that it is subject to an ``equal tails" assumption: it always outputs intervals that have approximately equal mass lying outside of the interval on each side.  This may not be volume-optimal if the data distribution is skewed.  

However, even if the data is subject to an equal tails constraint, where an input sequence must have each \(y_t\) drawn from a distribution \(\distribution_t\) that is unimodal and symmetric around median \(1/2\), the lower bound in \Cref{thm:arbitrary-order-lower-bound} still holds.  We can interpret this as saying that the quantile-based algorithm is not conservative enough to achieve a strong volume approximation.  
Put differently, suppose that on each day \(t\), the algorithm was told the distribution \(\distribution_t\) of \(y_t\) before outputting a confidence set for \(y_t\).  In this setting, choosing the volume optimal set that achieves coverage \(1 - \alpha\) on each day does not necessarily imply that the algorithm will compete in volume with the best set that achieves coverage \(1 - \alpha\) over the sequence of \(y_t\)s.  The optimum in hindsight may have lower coverage on days that have spread out distributions and higher coverage on days that have concentrated distributions.  This fact that the optimal loss over the sequence could be lower than the sum of the optimal losses on each day poses a roadblock to any strategy that goes through regret minimization with respect to a fixed loss function.  

\paragraph{"Best of both worlds" for Online Conformal Prediction.}  We note that the work of \citet*{ABB-best-of-both-worlds} provides a ``best of both worlds" guarantee for online conformal prediction.  This does not contradict our result that there is no possible ``best of both worlds" guarantee, as the two works are referring to different notions of optimality.  \citet*{ABB-best-of-both-worlds} provide a guarantee that simultaneously achieves average coverage tending to \(1 - \alpha\) on arbitrary sequences, and coverage on each day tending to \(1 - \alpha\) on exchangeable sequences.  In this work, we are interested in optimizing efficiency.  We consider algorithms that achieve average coverage both for arbitrary and exchangeable sequences, and we would like to show that they achieve the best possible volume approximation ratio in these two settings.  We show that this type of guarantee is impossible to achieve.

\section{Algorithmic Guarantees for Arbitrary Sequences}

We begin by analyzing the performance of \Cref{alg:meta-algorithm} on arbitrary sequences, for general settings of the allowable error rate \(R(t)\).  In the analysis, we show that, even though it is not explicitly enforced by the algorithm, most times that \(\Icurrent\) is reset (\Cref{alg:meta-algorithm}, Line \ref{algline:reset-condition}), its size increases multiplicatively by a factor that is linear in \(\mu\) (the core of this argument is illustrated in \Cref{fig:doubling-argument}).  Since the smallest width that \(\Icurrent\) is reset to is \(\minwidth\), and the largest it can be is \(1\), the number of times it is resets essentially depends on \(\log_\mu (1/\minwidth) = \log(1/\minwidth)/\log(\mu)\). 

Call the time window between two subsequent resets of \(\Icurrent\) a ``phase." Over a phase, \Cref{alg:meta-algorithm} plays one fixed interval that is feasible over the whole phase.  This allows us to bound the number of mistakes made in a particular phase by the total number of mistakes that any feasible interval is allowed to have made at that time.

\begin{theorem}[Meta-algorithm Guarantee for Arbitrary Sequences]\label{thm:meta-alg-arbitrary-order}
    \addcontentsline{toc}{subsection}{\emph{\Cref{thm:meta-alg-arbitrary-order}:} Meta-algorithm Guarantee for Arbitrary Sequences}
    Fix a scale lower bound \\
    \(\minwidth > 0\), a multiplicative volume approximation \(\mu > 3\), and an allowable error rate function \\
    \(R(t) : \{0, \dots, T - 1\} \rightarrow [0, 1]\) such that \(R(t) \le \frac{1}{4}\) for all \(t \ge t^\star\).

    On a sequence \(S = (y_1, \dots, y_T)\) of length \(T\), let \(I_t\) be the interval that \Cref{alg:meta-algorithm} plays on each day \(t \in [T]\).  The maximum volume of these intervals is bounded as
    \[\vol(I_t) ~\le~ \mu \max \left\{ \opt_S \left( \min_{1\le t \le T} \frac{t}{T} R(t) \right), ~\minwidth \right\} \qquad \forall t \in [T].\]
    The number of mistakes of these intervals can be bounded as 
    \[\sum_{i = 1}^T \mathbf{1}\{y_t \notin I_t\} ~\le~ t^\star + O \left(\frac{\log (1/\minwidth) \log T}{\log (\mu)} \left(1 + \max_{0 \le t \le T - 1} R(t) t\right) \right). \]
    Moreover, the algorithm is deterministic and therefore robust to an adaptive adversary. 
\end{theorem}

\begin{proof}
    Any interval that has made \(\le R(t - 1) (t - 1)\) mistakes before day \(t \ge 2\) is a feasible choice for \(I_t\) (\Cref{alg:meta-algorithm}, line \ref{algline:feasible-interval}).  Therefore, any interval \(I\) that satisfies 
    \[\frac{1}{T}\sum_{i = 1}^T \mathbf{1}\{y_t \notin I\} ~\le~ \min_{1 \le t \le T-1} \frac{t}{T} R(t)\]
    is a feasible choice for \(I_t\), on every day \(t\).  In particular, \(\opt_S \left( \min_{1 \le t < T} \frac{t}{T} R(t) \right)\) satisfies the above by definition, and is therefore always a feasible choice for \(I_t\).
    
    \anote{The above sentence isn't easy to read without a clear mathematical expression on the left. Would it be too cumbersome to have notation for denoting these quantities?}
    Lines \ref{algline:feasible-interval} and \ref{algline:expand-interval} ensure that the interval that is played has volume at most \(\mu\) times the maximum of \(\minwidth\) and the volume of any feasible interval.  Therefore the volume played by \Cref{alg:meta-algorithm} on each day \(t'\) is bounded by 
    \begin{equation*}
        \vol(I_{t'}) ~\le~ \mu \max \left\{ \opt_S \left( \min_{1 \le t < T} \frac{t}{T} R(t) \right), ~\minwidth \right\}, \qquad \forall t' \in [T]. 
    \end{equation*}

    Now we bound the number of mistakes.  Let \(t_1, t_2\) be two subsequent times that \(\Icurrent\) is reset.  That is, the same \(\Icurrent\) is played on all days \([t_1, t_2)\).  Call \([t_1, t_2)\) a \emph{phase}.  We begin by bounding the number of mistakes in a phase.  Since \(\Icurrent\) is feasible on day \(t_2 - 1\), \(\Icurrent\) makes \(\le R(t_2 - 2) (t_2 - 2)\) mistakes on \(\{y_1, \dots, y_{(t_2 - 1)}\}\).  Therefore \(\Icurrent\) makes \(\le R(t_2 - 2) (t_2 - 2) + 1\) mistakes on this phase: \(\{y_{t_1}, \dots, y_{t_2 - 1}\}\).  Thus the number of mistakes that the algorithm makes on any phase is
    \begin{equation}
        \sum_{t = t_1}^{t_2 - 1}\mathbf{1}\{y_t \in \Icurrent\} \le 1 + \max_{0 \le t \le T - 1} R(t)t. \label{eq:mistakes-per-phase}
    \end{equation}
    \vnote{rename \(x\)s to \(y\)s!}
    
    Now we bound the number of phases.
    Consider any two days \(t_1, t_2\) such that \(t^\star < t_1 < t_2\), so \(R(t_1 - 1), R(t_2 - 2) \le \frac{1}{4}\).  Let \(I_1\) be any interval that is feasible at \(t_1\) and \(I_2\) be any interval that is feasible at \(t_2\).  \(I_1\) has captured \(\ge \frac{3}{4} (t_1 - 1)\) points out of \(\{x_1, \dots, x_{t_1 - 1}\}\).  \(I_2\) has missed \(\le \frac{1}{4} (t_2 - 1)\) points out of \(\{y_1, \dots, y_{t_1 - 1}\} \subseteq \{y_1, \dots, y_{t_2 - 1}\}\).  Therefore as long as \(\frac{t_2 - 1}{t_1 - 1} < 3\), \(I_1\) and \(I_2\) must capture some point in common and must overlap.  

    Partition the sequence after day \(t^\star\) into \(O(\log T)\) \emph{epochs}, where each epoch \([t_\mathrm{start}, t_\mathrm{end}]\) has \(\frac{t_\mathrm{end} - 1}{t_\mathrm{start} - 1} < 3\).    
    We bound the number of mistakes that \Cref{alg:meta-algorithm} makes on each epoch.  Fix an epoch \(e\), and consider the first time in this epoch that \Cref{alg:meta-algorithm} resets \(\Icurrent\).  By line \ref{algline:expand-interval}, \(\Icurrent\) is set to have volume \(\ge \mu \minwidth\). 

    Now consider a subsequent time \(t\) in epoch \(e\) when \(\Icurrent\) is reset.  Let \(\tprev < t\) be the previous time in epoch \(e\) when \(\Icurrent\) was reset.  Let \(I_\tprev\) and \(I_{t}\) be the minimum volume feasible intervals on day \(\tprev\) and \(t\), respectively (as chosen in line  \ref{algline:feasible-interval}).  Let \(\Icurrent^{(\tprev)}\) denote the state of \(\Icurrent\) before day \(t\) (after day \(\tprev\)).
    
    Since \(\Icurrent\) is reset on day \(t\), \(\Icurrent^{(\tprev)}\) is not feasible on day \(t\), and in particular \(I_t \not\subseteq \Icurrent^{(\tprev)}\).  By the previous argument \(I_\tprev\) and \(I_t\) must overlap.  Any interval that contains a point in \(I_\tprev\) and also contains a point that is not in \(\Icurrent^{(\tprev)} = \widehat{\mu} I_\tprev\), must have volume \(\ge \frac{\mu - 1}{2} \cdot  \vol \left(\Icurrent^{(\tprev)} \right)\).  This argument is illustrated in \Cref{fig:doubling-argument}. Denote the state of \(\Icurrent\) at the end of day \(t\) as \(\Icurrent^{(t)}\).  By line \ref{algline:expand-interval}, we have that \(\Icurrent^{(t)} \supseteq \mu I_t\), so 
    \begin{equation*}
        \vol\left(\Icurrent^{(t)} \right) \ge \frac{\mu - 1}{2} \vol\left(\Icurrent^{(\tprev)}\right) .
    \end{equation*}
    Since \(\vol(I_t) \le 1\) for all days \(t\), we have that \(\vol(\Icurrent^{(t)}) \le \mu\) for all days \(t\).  Thus, within an epoch, as long as \(\mu > 3\), the number of times \(\Icurrent\) can be reset is
    \begin{equation}
        \le 1 + \log_{\frac{\mu - 1}{2}} \left( \frac{\mu}{\mu \minwidth} \right) 
        = O\left( \frac{\log (1/\minwidth)}{\log (\mu)} \right), \label{eq:resets-in-an-epoch}
    \end{equation}
    and the total number of times that \(\Icurrent\) is reset after day \(t^\star\) is \(O\left( \frac{\log (1/\minwidth) \log T}{\log (\mu)}\right)\).

    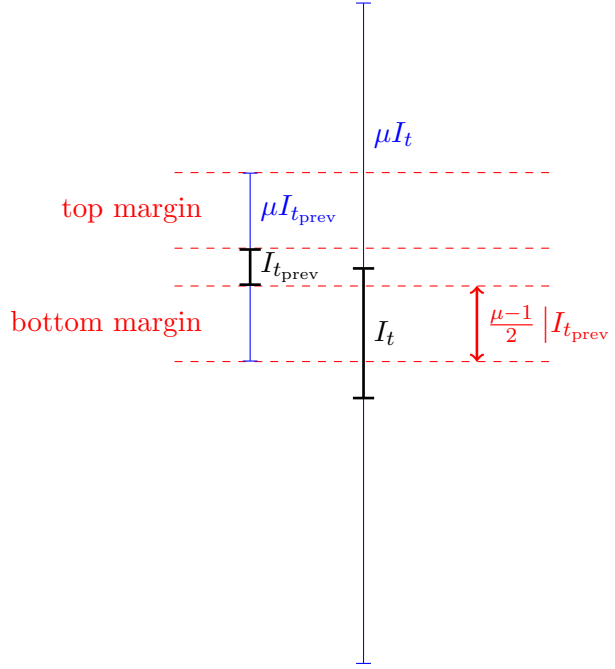
\begin{figure}[t]
        \centering
        \begin{tikzpicture}

            \draw[dashed, red] (0, 3.5) -- (5, 3.5);
            \draw[dashed, red] (0, 4.5) -- (5, 4.5);
            \node[anchor=east, red] (q) at (0.5, 4) {top margin};
    
            \draw[dashed, red] (0, 3) -- (5, 3);
            \draw[dashed, red] (0, 2) -- (5, 2);
            \node[anchor=east, red] (q) at (0.5, 2.5) {bottom margin};
    
            \draw[|-|, blue] (1, 2) -- (1, 4.5);
            \node[anchor=west, blue] (b) at (1, 4) {$\mu I_{t_\mathrm{prev}}$};
            \draw[|-|, line width = 1pt] (1, 3) -- (1, 3.5);
            \node[anchor=west] (a) at (1, 3.25) {$I_{t_\mathrm{prev}}$};
    
            \draw[|-|, blue] (2.5, -2) -- (2.5, 6.75);
            \node[anchor=west, blue] (b) at (2.5, 5) {$\mu I_{t}$};
            \draw[|-|, line width = 1pt] (2.5, 1.5) -- (2.5, 3.25);
            \node[anchor=west] (a) at (2.5, 2.375) {$I_{t}$};
    
            \draw[<->, line width = 1pt, red] (4, 2) -- (4, 3);
            \node[anchor=west, red] (d) at (4, 2.5) {$\frac{\mu - 1}{2} \left|I_{t_\mathrm{prev}} \right|$};
        \end{tikzpicture}
        \caption{\(I_t\) achieved more coverage than \(\Icurrent^{(\tprev)} = \mu I_{\tprev}\) on the sequence so far, so \(I_t\) is not contained in \(\mu I_{\tprev}\).
        We also know that \(I_{\tprev}\) and \(I_{t}\) must overlap.  
        Thus, \(I_t\) must fully contain one of the margins that we add to \(I_{\tprev}\) to form \(\mu I_{\tprev}\) (denoted here as the regions between the red dashed lines).  This gives a lower bound on the size of \(I_t\) in terms of the size of \(I_{\tprev}\), and therefore a lower bound on the size of \(\Icurrent^{(t)}\) in terms of the size of \(\Icurrent^{(\tprev)}\)
        }
        \label{fig:doubling-argument}
    \end{figure}
    
    Our previous bound on the number of mistakes in a phase (\ref{eq:mistakes-per-phase}) allows us to bound the total number of mistakes that \Cref{alg:meta-algorithm} makes as
    \begin{equation*}
        = t^\star + O \left(\frac{\log (1/\minwidth) \log T}{\log (\mu)} \left(1 + \max_{0 \le t \le T - 1} R(t) t\right) \right). 
    \end{equation*}
\end{proof}

The previous analysis allows us to choose an \(R(t)\) to optimize the bounds.  The following corollary observes that \(R(t) = \frac{\alpha T}{t}\) gives a guarantee in terms of \(\opt(\alpha)\).  This gives a guarantee that exhibits a tradeoff between \(\mu\), the multiplicative volume approximation of the intervals played by \Cref{alg:meta-algorithm}, and the number of mistakes that \Cref{alg:meta-algorithm} makes.  Later, in \Cref{thm:arbitrary-order-lower-bound} we give a lower bound that shows that the \(\mu\)-mistake tradeoff that \Cref{alg:meta-algorithm} with \(R(t) = \frac{\alpha T}{t}\) achieves is optimal.  That is, for any setting of \(\mu\), this algorithm makes a near-optimal number of mistakes subject to achieving multiplicative approximation \(\le \mu\). 

\(R(t) = \frac{\alpha T}{t}\) ensures that the intervals that are feasible on day \(t\) are exactly those that have made \(\le \alpha T\) mistakes so far, where \(\alpha T\) is the total number of mistakes that \(\opt(\alpha)\) is allowed to make over the whole sequence.  We can think of this algorithm as playing conservatively, and not ruling out any interval until it is demonstrably not \(\opt(\alpha)\).  

\begin{corollary}[Algorithm for Arbitrary Sequences]\label{cor:optimal-algorithm-for-arbitrary-order}
    \addcontentsline{toc}{subsection}{\emph{\Cref{cor:optimal-algorithm-for-arbitrary-order}}: Algorithm for Arbitrary Sequences}
    Fix a scale lower bound \(\minwidth > 0\), multiplicative volume approximation \(\mu > 3\), target miscoverage rate \(\alpha \ge 0\), and time horizon \(T\) (all known to the algorithm).  Set \(R(t)\) such that \(R(0) = 1\) and \(R(t) = \alpha \frac{T}{t}\) for \(t > 0\).  With these parameters, on any sequence \(S = (y_1, \dots, y_T)\) of length \(T\), \Cref{alg:meta-algorithm} plays interval \(I_t\) on each day \(t \in [T]\), such that the maximum volume of any \(I_t\) is approximately optimal 
    \[\vol(I_t) \le \mu \max \left\{ \opt_S ( \alpha ), ~\minwidth \right\} \qquad \forall t \in [T],\]
    and the number of mistakes is bounded by 
    \[\sum_{i = 1}^T \mathbf{1}\{ y_t \notin I_t \} = O \left( \frac{\log (1/\minwidth)}{\log (\mu)} (\alpha T + 1) \right).\]
    Moreover, the algorithm is deterministic and therefore robust to an adaptive adversary. 
\end{corollary}

\begin{proof}
    The bound on maximum volume follows directly from the volume guarantee in \Cref{thm:meta-alg-arbitrary-order}.  

    For the bound on number of mistakes, this choice of \(R(t)\) ensures that an interval is feasible on day \(t\) if and only if it has made \(\le \alpha T\) mistakes up to day \(t\).  Consider \(t_1, t_2\) such that \(2 \alpha T + 1 < t_1 < t_2\).  Let \(I_1\) be any feasible interval on day \(t_1\) and \(I_2\) be any feasible interval on day \(t_2\).  Since \(I_1\) and \(I_2\) each make \(\le \alpha T\) mistakes over the first \(2 \alpha T + 1\) days, they must capture at least one day in common.  Therefore \(I_1\) and \(I_2\) overlap.  Thus we can think of all days \(t > 2 \alpha T + 1\) as forming one epoch, and the proof of \Cref{thm:meta-alg-arbitrary-order} (\Cref{eq:resets-in-an-epoch}) tells us that the number of phases in this epoch is 
    \[ = O\left( \frac{\log (1/\minwidth)}{\log (\mu)} \right).\]
    The proof of \Cref{thm:meta-alg-arbitrary-order} (\Cref{eq:mistakes-per-phase}) also tells us that, for this choice of \(R(t)\), the number of mistakes in each phase is bounded by 
    \(\le 1 + \alpha T.\)
    Thus the total number of mistakes made is 
    \[ \le (2 \alpha T + 1) + O\left( \frac{\log (1/\minwidth)}{\log (\mu)} \right) (\alpha T + 1) \quad = O\left( \frac{\log (1/\minwidth)}{\log (\mu)} (\alpha T + 1)\right) .\]
\end{proof}

\section{Algorithmic Guarantees for Exchangeable Sequences}

In this section we analyze the performance of \Cref{alg:meta-algorithm} on exchangeable sequences.  The analysis differs significantly from the arbitrary order setting, because we want to take advantage of concentration.  That is, an exchangeable sequence \(\mathbf{S}\) has the property that the coverage of any interval \(I \subseteq [0,1 ]\) over one subsequence of \(\mathbf{S}\) should be similar to the coverage of \(I\) over a different subsequence of \(\mathbf{S}\).  

Quantitatively, we can access this property through \emph{uniform convergence}.  In the most commonly used form, uniform convergence tells us that for a set family \(\mathcal{F}\) of bounded VC-dimension, for us the family of intervals over \([0, 1]\), and a set \(Y\) of \(n\) samples drawn i.i.d. from some distribution \(\distribution\), the coverage of every set in \(\mathcal{F}\) over \(Y\) simultaneously converges to its true coverage over \(\distribution\).  That is, treating the VC-dimension of \(\mathcal{F}\) and failure probability as constants, the coverage of every set in \(\mathcal{F}\) over \(Y\) is within \(\pm O (1/\sqrt{n})\) of its coverage over \(\distribution\).

A similar statement holds even for sets of samples that are exchangeable but not necessarily i.i.d..  Let \(\mathbf{S} = (\mathbf{S}_1, \dots, \mathbf{S}_t)\) be an exchangeable sequence of length \(T\), and let \(Y\) be some (not necessarily contiguous) subsequence of size \(t \le T\) of \(\mathbf{S}\).

Uniform convergence for exchangeable sequences tells us that, treating the VC-dimension of \(\mathcal{F}\) and failure probability as constants, the coverage of every set in \(\mathcal{F}\) over \(Y\) is within \(\pm O (1/\sqrt{t})\) of its coverage over \(\mathbf{S}\).  That is, the coverage of sets in \(\mathcal{F}\) over subsequences of \(\mathbf{S}\) behaves essentially as though the elements of \(Y\) were drawn i.i.d.\ from \(\unif \{\mathbf{S}_1, \dots, \mathbf{S}_T \}\).  A formal version of this statement can be found in \Cref{lem:exchangeable-uniform-convergence}.  

In analyzing the algorithm on exchangeable sequences, we must bound the number of mistakes by arguing that on each day \(t\), the interval \(I_t\) chosen by the algorithm has high coverage.  However, this requires us to have a tight bound on the coverage of \(I_t\) over the subsequence of length 1 that only contains day \(t\), and the previous argument does not give tight bounds for short subsequences.  We address this issue by instead bounding the coverage of \(I_t\) on the \emph{suffix} of the input sequence that contains all days between \(t\) and \(T\).  Considering a longer subsequence allows us to get tighter bounds.  Then, we observe that the coverage of \(I_t\) on the suffix is a property only of the unordered (multi)set of values in the suffix, and not the order that the values appear.  Thus, conditioned on the coverage bound holding, the suffix of the sequence is still exchangeable within itself, so the coverage on day \(t\) is actually equal to the coverage over the whole suffix.  

\begin{theorem}[Meta-algorithm Guarantee for Exchangeable Sequences]\label{thm:meta-alg-exchangeable}
    \addcontentsline{toc}{subsection}{\emph{\Cref{thm:meta-alg-exchangeable}:} Meta-algorithm Guarantee for Exchangeable Sequences}
    Fix a scale lower bound \\
    \({\minwidth > 0}\), a multiplicative volume approximation \(\mu \ge 1\), and an allowable error rate function \(R(t): \{0, \dots, T - 1\} \rightarrow [0, 1]\).  Let \(\mathbf{S}\) be an \emph{exchangeable} input sequence of length \(T\).  
    Let \(I_t\) be a random variable, which is the interval that \Cref{alg:meta-algorithm} plays on input \(\mathbf{S}\).
    Then, the expected maximum volume of the intervals can be bounded as 
    \[\E\left[\max_t \vol(I_t)\right] ~\le~ \mu \max \left\{ \opt_\mathbf{S} \left( \min_{0 \le t < T} \left[R(t) - O\left(\sqrt{\log T / t}\right) \right] \right), ~ \minwidth \right\} + \frac{1}{T^2},\]
    and achieves expected coverage 
    \[\E\left[\frac{1}{T} \mathbf{1}\{y_t \in I_t \}\right] ~\ge~ 1 - \frac{1}{T} \sum_{t = 1}^T R(t - 1) - O\left( \sqrt{\log T / T} \right).\]
\end{theorem}

\begin{proof}
    Fix an exchangeable input sequence \(\mathbf{S} = (Y_1, \dots, Y_T)\), where the random variable \(Y_i\) is the input on day \(i\).
    Let the random variable \(M^{\mathbf{S}}\) be the number of mistakes that \Cref{alg:meta-algorithm} (for \(\minwidth\), \(\mu\), \(R(t)\)) makes on \(\mathbf{S}\).  Let \(M^{\mathbf{S}}_t\) be an indicator random variable of whether \Cref{alg:meta-algorithm} makes a mistake on day \(t\).  That is, let \(I_t\) be the interval that \Cref{alg:meta-algorithm} plays on day \(t\).  Then 
    \[M_t^{\mathbf{S}} = \mathbf{1} \left[Y_t \notin I_t \right].\]

    We analyze the algorithm's performance on one particular day \(t\).  We will refer to the days before \(t\), i.e.\ the window \([1, t-1]\), as the \emph{prefix}, and the days \(t\) and after, i.e.\ the window \([t, T]\), as the \emph{suffix}.  
    For an interval \(I \subseteq [0, 1]\) and a window (contiguous subsequence) \(W = [t_1, t_2] \subseteq [1, T]\), define \(\error_{\mathbf{S}}(I, W)\) to be the fraction of days in \(W\) of \(\mathbf{S}\) that are not covered by \(I\).
    Since the set of intervals over \([0, 1]\) has VC-dimension 2, \Cref{lem:exchangeable-uniform-convergence} tells us that the following statements hold, each with probability \(\ge 1 - \frac{1}{T^3}\), for some absolute constant \(C\):
    \begin{align}
        \forall I \quad \error_{\mathbf{S}}(I, [1, T]) - C \sqrt{\frac{\log T}{t - 1}} ~\le~ \error_{\mathbf{S}}(I, [1, t-1])  ~&\le~ \error_{\mathbf{S}}(I, [1, T])+ C \sqrt{\frac{\log T}{t - 1}}, \label{eq:UC-prefix}\\
        \forall I \quad \error_{\mathbf{S}}(I, [1, T-t + 1]) - C \sqrt{\frac{\log T}{t - 1}} ~\le~ \error_{\mathbf{S}}(I, [t, T])  ~&\le~ \error_{\mathbf{S}}(I, [1, T])+ C \sqrt{\frac{\log T}{T - t + 1}}. \label{eq:UC-suffix}
    \end{align}
    We will refer to (\ref{eq:UC-prefix}) and (\ref{eq:UC-suffix}) as the \emph{uniform convergence property}.  Note that (\ref{eq:UC-prefix}) and (\ref{eq:UC-suffix}) are properties only of the set of values that appear in the prefix and suffix, and does not imply anything about the order that elements appear in the suffix.  That is, conditioned on (\ref{eq:UC-prefix}) and (\ref{eq:UC-suffix}), the suffix is still exchangeable, and we have
    \begin{equation}
        \forall I, t' \in [t, T] \quad \ProbOp_{\mathbf{S}} \left[ Y_{t'} \in I ~|~ \text{(\ref{eq:UC-prefix}), (\ref{eq:UC-suffix})}\right] = \E_{\mathbf{S}} \left[ \error_{\mathbf{S}}(I, [t, T]) ~|~ \text{(\ref{eq:UC-prefix}), (\ref{eq:UC-suffix})} \right]. \label{eq:conditioning-preserves-exchangeability}
    \end{equation}
    This allows us to bound the expected value of \(M^{\mathbf{S}}_i\) conditioned on (\ref{eq:UC-prefix}) and (\ref{eq:UC-suffix}).  
    We have deterministically that \(\error_{\mathbf{S}} \left(I_t, [1, t-1] \right) \le R(t - 1)\) (\Cref{alg:meta-algorithm}, Line \ref{algline:feasible-interval}). 
    \begin{align*}
       \E_{\mathbf{S}} \left[ M^{\mathbf{S}}_t ~|~ \text{(\ref{eq:UC-prefix}), (\ref{eq:UC-suffix})} \right] &= \ProbOp_{\mathbf{S}} \left[ Y_t \in I_t ~|~ (\ref{eq:UC-prefix}), (\ref{eq:UC-suffix})  \right] \\
       &= \E_{\mathbf{S}} \left[ \error_{\mathbf{S}}(I, [t, T]) ~|~ \text{(\ref{eq:UC-prefix}), (\ref{eq:UC-suffix})} \right] &\text{by (\ref{eq:conditioning-preserves-exchangeability})} \\
       &\le \E_{\mathbf{S}} \left[ \error_{\mathbf{S}}(I, [1, T]) +  C \sqrt{\frac{\log T}{T - t + 1}} ~\bigg|~ \text{(\ref{eq:UC-prefix}), (\ref{eq:UC-suffix})} \right] &\text{by (\ref{eq:UC-suffix})} \\
       &\le \E_{\mathbf{S}} \left[ \error_{\mathbf{S}}(I, [1, t-1]) +  C \sqrt{\frac{\log T}{T - t + 1}} +  C \sqrt{\frac{\log T}{t - 1}} ~\bigg|~ \text{(\ref{eq:UC-prefix}), (\ref{eq:UC-suffix})} \right] &\text{by (\ref{eq:UC-prefix})} \\
       &\le R(t - 1) +  C \sqrt{\frac{\log T}{T - t + 1}} +  C \sqrt{\frac{\log T}{t - 1}}  .
    \end{align*}
    We also have that \(M_t^{\mathbf{S}} \le 1\).  So in total this gives us that 
    \begin{equation}
        \E_{\mathbf{S}} \left[ M^{\mathbf{S}}_t \right] \le R(t - 1) +  C \sqrt{\frac{\log T}{T - t + 1}} +  C \sqrt{\frac{\log T}{t - 1}} + \frac{2}{T^3}. \label{eq:single-day-mistake-bound}
    \end{equation}
    This allows us to bound \(M^{\mathbf{S}}\). 
    \begin{align*}
        \E_{\mathbf{S}} \left[ M^{\mathbf{S}} \right] &\le \sum_{t = 1}^T \left[ M^{\mathbf{S}}_t \right] \\
        &\le \sum_{t = 1}^T  \left[ R(t - 1) +  C \sqrt{\frac{\log T}{T - t + 1}} +  C \sqrt{\frac{\log T}{t - 1}} + \frac{2}{T^3} \right] \\
        &\le O \left(\sqrt{T \log T} \right) + \sum_{t = 1}^T R(t - 1).
    \end{align*}
    Thus the expected coverage of \Cref{alg:meta-algorithm} can be bounded as 
    \begin{equation*}
        1 - \frac{1}{T} \E_{\mathbf{S}} \left[M^{\mathbf{S}} \right] = 1 - \frac{1}{T}\sum_{t = 1}^T R(t - 1) - O\left( \sqrt{\log T / T} \right).
    \end{equation*}

    Now we bound the volume of intervals played by \Cref{alg:meta-algorithm}.  We condition on (\ref{eq:UC-prefix}) holding simultaneously for all \(t \in [1, T]\).  Since this holds with probability \(\ge 1 - \frac{1}{T^3}\) for each \(t\), it holds for all \(t\) with probability \(\ge 1 - \frac{1}{T^2}\).  

    Consider any interval \(I\) that has coverage \(\ge 1 - \min_{0 \le t \le T} \left[ R(t) + C \sqrt{\log T / t} \right]\) over \(\mathbf{S}\) in hindsight.  Then, conditioned on (\ref{eq:UC-prefix}) holding for all \(t\), \(I\) must be a feasible choice for \(I_t\) on every day \(t\) (\Cref{alg:meta-algorithm}, Line \ref{algline:feasible-interval}).  Thus, conditioned on (\ref{eq:UC-prefix}) holding for all \(t\), \Cref{alg:meta-algorithm} plays intervals of maximum volume  
    \[ \le \mu \max \left\{ \opt_\mathbf{S} \left( \min_{0 \le t < T} \left[R(t) - O\left(\sqrt{\log T / t}\right) \right] \right), ~ \minwidth \right\}. \]  
    Even if (\ref{eq:UC-prefix}) does not hold for all \(t\), \Cref{alg:meta-algorithm} plays intervals of maximum length 1.  So \Cref{alg:meta-algorithm} plays intervals of expected maximum volume 
    \[\le \mu \max \left\{ \opt_\mathbf{S} \left( \min_{0 \le t < T} \left[R(t) - O\left(\sqrt{\log T / t}\right) \right] \right), ~ \minwidth \right\} ~+ \frac{1}{T^2}.\]
\end{proof}

From the previous guarantee, it is simple to derive a setting of \(R(t)\) that achieves volume essentially that of \(\opt(\alpha)\), and coverage that is within \(O(\sqrt{\log T/T})\) of the target \(1 - \alpha\).  The \(O(\sqrt{\log T/T})\) term is a standard sampling error that should arise since exchangeable sequences are only more general than i.i.d.\ sequences. 

\begin{corollary}[Algorithm for Exchangeable Sequences] \label{cor:optimal-algorithm-for-exchangeable-sequences}
    \addcontentsline{toc}{subsection}{\emph{\Cref{cor:optimal-algorithm-for-exchangeable-sequences}:} Algorithm for Exchangeable Sequences}
    Fix a scale lower bound \(\minwidth > 0\), miscoverage rate \(\alpha\), and time horizon \(T\) (all known to the algorithm).  Set \(\mu = 1\), and \(R(t) = \alpha + O\left(\sqrt{\log T/t} \right)\).  Let \(\mathbf{S}\) be an exchangeable input sequence of length \(T\).  Let \(I_t\) (a random variable) be the interval played by \Cref{alg:meta-algorithm}, with the above parameters, on input \(\mathbf{S}\).
    The expected maximum volume of the \(I_t\) can be bounded as 
    \[\E\left[ \max_t \vol(I_t) \right] ~\le~ \max \{ \opt_{\mathbf{S}} (\alpha), ~ \minwidth \} + \frac{1}{T^2},\]
    and the algorithm achieves expected coverage 
    \[\E \left[\frac{1}{T} \sum_{i = 1}^T \mathbf{1} \{y_t \in I_t \} \right] \ge (1 - \alpha) - O\left(\sqrt{\log T / T} \right).\]
\end{corollary}

\begin{proof}
    Follows directly from \Cref{thm:meta-alg-exchangeable}.
\end{proof}

We observe that this near-optimal algorithm for online conformal prediction on exchangeable sequences implies a near-optimal algorithm for standard conformal prediction with exchangeable data.  We view this simple observation as evidence that online conformal prediction on exchangeable sequences can be seen as a generalization of standard conformal prediction with exchangeable data, in a similar way to how online learning generalizes PAC learning. 

We note this statement is also a direct consequence of the result of \citet*{gao2025volumeoptimalityconformalprediction}, which states this for \(Y_i\) drawn i.i.d.\ from some unknown distribution \(\distribution\), along with uniform convergence for exchangeable sequences (\Cref{lem:exchangeable-uniform-convergence}), via a similar sample splitting argument.  
\vnote{please check this statement!}

\begin{corollary}[Efficiency Guarantee for Standard Conformal Prediction]\label{cor:efficiency-for-standard-conformal-prediction}
    \addcontentsline{toc}{subsection}{\emph{\Cref{cor:efficiency-for-standard-conformal-prediction}}: Efficiency Guarantee for Standard Conformal Prediction}
    Fix a scale lower bound \(\minwidth > 0\), and miscoverage rate \(\alpha\).  Let \(\mathbf{S} = (Y_1, \dots, Y_{T})\) be exchangeable random variables valued in \([0, 1]\).  

    Given \(Y_1, \dots, Y_{T-1}\), we construct the following conformal predictor for \(Y_T\). 
    \anote{Add: Suppose } 
    Suppose we run \Cref{alg:meta-algorithm} with parameters \(\minwidth, \mu = 1, \alpha, T, R(t) = \alpha + O(\sqrt{\log T / t})\) for \(\lceil t/2 \rceil\) days using \(Y_1, \dots, Y_{\lceil t/2 \rceil}\) as inputs.  Let \(I_{\lceil t/2 \rceil}\) be the interval that \Cref{alg:meta-algorithm} predicts on day \(\lceil t/2 \rceil\).  
\anote{The statement is long. Here say "Then..."}

    Then, \(I_{\lceil t/2 \rceil}\) is a conformal predictor for \(Y_T\), achieving coverage 
    \[\ProbOp_{\mathbf{S}} \left[ Y_T \in I_{\lceil t/2 \rceil} \right] \ge (1 - \alpha) - O \left( \sqrt{\frac{\log T}{T}} \right),\]
    and \(I_{\lceil t/2 \rceil}\) has volume bounded by 
    \[\vol(I_{\lceil t/2 \rceil}) ~\le~ \max \{ \opt_{\mathbf{S}} (\alpha), ~ \minwidth \} + \frac{1}{T^2}.\]
\end{corollary}

\begin{proof}
    This follows from the proof of \Cref{thm:meta-alg-exchangeable}.
    Let \(I_{\lceil t/2 \rceil}\) be the interval that \(\Cref{alg:meta-algorithm}\) plays on day \(\lceil t/2 \rceil\) of input \(\mathbf{S}\).
    Let \(M_{\mathrm{final}}^{\mathbf{S}}\) be an indicator random variable of whether \(Y_{T} \notin I_{\lceil t/2 \rceil} \).  
    By (\ref{eq:conditioning-preserves-exchangeability}) we have that 
    \[\E_{\mathbf{S}} \left[ M_{\mathrm{final}}^{\mathbf{S}} ~|~ \text{(\ref{eq:UC-prefix}) and (\ref{eq:UC-suffix})} \right] = 
    \ProbOp_{\mathbf{S}} \left[ Y_{T} \in I_{\lceil t/2 \rceil} ~|~ \text{(\ref{eq:UC-prefix}) and (\ref{eq:UC-suffix})} \right] = \ProbOp_{\mathbf{S}} \left[ Y_{\lceil t/2 \rceil} \in I_{\lceil t/2 \rceil} ~|~ \text{(\ref{eq:UC-prefix}) and (\ref{eq:UC-suffix})} \right].\]
    Thus (\ref{eq:single-day-mistake-bound}) ensures that 
    \begin{align*}
        \ProbOp_{\mathbf{S}} \left[ Y_{T} \notin I_{\lceil t/2 \rceil} \right] &= \E_{\mathbf{S}} \left[ M_{\mathrm{final}}^{\mathbf{S}} \right] \\
        &\le R(\lceil t/2\rceil - 1) + C \sqrt{\frac{\log T}{T - \lceil t/2 \rceil + 1}} + C \sqrt{\frac{\log T}{\lceil t/2 \rceil - 1}} + \frac{2}{T^3} \\
        &\le \alpha + O \left( \sqrt{\frac{\log T}{T}} \right),
    \end{align*}
    and \(I_{\lceil t/2 \rceil}\) achieves coverage 
    \[\ge (1 - \alpha) - O \left( \sqrt{\frac{\log T}{T}} \right)\]
    on \(Y_{T}\).  

    For the efficiency bound, \Cref{cor:optimal-algorithm-for-exchangeable-sequences} ensures that the expected maximum volume of any interval played by \Cref{alg:meta-algorithm}, and therefore the expected volume of \(I_{\lceil t/2 \rceil}\) is bounded by 
    \[\le \max \{ \opt_{\mathbf{S}} (\alpha), ~ \minwidth \} + \frac{1}{T^2}.\]
\end{proof}

Finally, we note that for the optimal algorithm for exchangeable sequences we use \Cref{alg:meta-algorithm} with \(R(t) = \alpha - O (\sqrt{\log T / T})\) (\Cref{cor:optimal-algorithm-for-exchangeable-sequences}).  However, for the optimal algorithm for arbitrary sequences we use \Cref{alg:meta-algorithm} with \(R(t) = \alpha T/t\) (\Cref{cor:optimal-algorithm-for-arbitrary-order}).  This raises the natural question of whether we can design a single algorithm (perhaps but not necessarily by using \Cref{alg:meta-algorithm} with a unified choice of \(R(t)\)) that simultaneously achieves the optimal guarantee on both exchangeable and arbitrary sequences.  

Later in \Cref{thm:no-best-of-both-worlds}, we give a lower bound that shows that it is indeed not possible to design any algorithm that is simultaneously optimal in both of these settings.  \Cref{thm:no-best-of-both-worlds} quantifies this bound by showing that for a given choice of \(\mu\), \(\minwidth\), and \(\alpha\), any algorithm that has volume bounded by \(\mu \max \{ \opt_S(\alpha), \minwidth \}\) on every \emph{arbitrary} input sequence \(S\), must make 
\[\widetilde\Omega \left( \min \left\{ \ln(1/\alpha), ~\frac{\ln (1/\minwidth)}{\ln \mu } \right\} \alpha T \right) \]
mistakes in expectation on some \emph{i.i.d.}\ sequence \(S'\).  Here \(\widetilde\Omega\) hides subpolynomial factors in \(\alpha\).  

Algorithmically, we show that \Cref{alg:meta-algorithm} with the conservative choice of \(R(t) = \alpha T / t\) is able to meet this mistake bound on exchangeable data, up to a sampling error term.  
We note that this algorithm has the same multiplicative volume approximation \(\mu\) on both arbitrary and exchangeable data, while our lower bound does not rule out the possibility of an algorithm having a better approximation ratio on exchangeable sequences than on arbitrary sequences. 

To match the lower bound, we must do a finer analysis of \Cref{alg:meta-algorithm} on exchangeable sequences that accounts for the number of phases of the algorithm.  In the analysis of \Cref{alg:meta-algorithm} on arbitrary sequences (\Cref{thm:meta-alg-arbitrary-order}), we used the fact that the algorithm plays one fixed interval for a whole phase.  This allowed us to amortize the error of the interval that the algorithm plays over the whole phase, and get a mistake bound that depends on the number of phases which is \(O(\log(1/\minwidth)/\log \mu)\).  We observe that when the number of phases is small, a similar amortizing argument can get a tighter bound on the number of mistakes that \Cref{alg:meta-algorithm} makes, even in the exchangeable case.  

\vnote{the \(\sqrt{\frac{\log(1/\minwidth)}{\log \mu}}\) is annoying, but I am convinced it is correct, and not too salient since it's upper bounded by the other case.}

\begin{theorem}[Simultaneous Guarantee for Arbitrary and Exchangeable Sequences]\label{thm:best-of-both-worlds-algorithm}
    \addcontentsline{toc}{subsection}{\emph{\Cref{thm:best-of-both-worlds-algorithm}:} Simultaneous Guarantee for Arbitrary and Exchangeable Sequences}
    Fix a scale lower bound \(\minwidth > 0\), target miscoverage rate \(\alpha\), multiplicative approximation \(\mu\), and time horizon \(T\) (all known to the algorithm).  Set \(R(t) = \frac{\alpha T}{t}\).  Let \(\mathbf{S}\) be an exchangeable input sequence of length \(T\).  On input \(\mathbf{S}\), \Cref{alg:meta-algorithm} with the above parameters makes 
    \[O \left( \min \left\{ \sqrt{T\log T} \sqrt{\frac{\log(1/\minwidth)}{\log \mu}} +  \frac{\log(1/\minwidth)}{\log \mu} \alpha T, \quad \sqrt{T \log T} + \log (1/\alpha)\alpha T \right\} \right)\]
    mistakes in expectation, and plays intervals of maximum volume 
    \[ \le \mu ~\max\left\{\opt_{\mathbf{S}}(\alpha), ~\minwidth \right\}.\]
\end{theorem}

\begin{proof}
    We condition on uniform convergence for every window (contiguous subsequence) of \(\mathbf{S}\).  For a window \(W = [t_1, t_2] \subseteq [1, T]\), and an interval \(I \subseteq[0, 1]\), define \(\error_{\mathbf{S}}(I, W)\) to be the fraction of days in \(W\) of \(\mathbf{S}\) that are not covered by \(I\).
    For each window \(W\), we have that with probability \(\ge 1 - \frac{1}{T^4}\), the coverage of every interval \(I\) over \(W\) is within 
    \[ \varepsilon_W = C \sqrt{\frac{\log T}{|W|}}\]
    its coverage over all of \(\mathbf{S}\), for some universal constant \(C\).
    Taking a union bound over all \(O(T^2)\) possible windows, we have that  
    \begin{equation}
        \forall W , I \quad \error_S(I, [1, T]) - \varepsilon_W \le \error_S(I, W) \le \error_S(I, [1, T]) + \varepsilon_W. \label{eq:conservative-uc-property}
    \end{equation}
    with probability \(\ge 1 - \frac{1}{T^2}\). 
    We refer to the property (\ref{eq:conservative-uc-property}) as the \emph{uniform convergence property}.

    Now let \(\widehat{S}\) be any realization of \(\mathbf{S}\) that satisfies the uniform convergence property. Let \(K\) be the number of times sets \(\Icurrent\) \Cref{alg:meta-algorithm} (for \(\mu\), \(\minwidth\), \(R(t) = \frac{\alpha T}{t}\)) on \(\widehat{S}\).  Let \(t_i\) for \(1 \le i \le K\) be the days when \(\Icurrent\) is set, where \(t_1 = 1\), and we define \(t_{K+1} = T+1\) for convenience.   We refer to the window \([t_i, t_{i-1})\) as \emph{phase \(i\)}.

    Let \(M^{\widehat{S}}\) be the number of mistakes that \Cref{alg:meta-algorithm} makes on \(\widehat{S}\).  Let \(M^{\widehat{S}}_i\) be the miscoverage rate of \Cref{alg:meta-algorithm} over phase \(i\) of \(\widehat{S}\).  
    Let \(I_i\) be the value of \(\Icurrent\) on phase \(i\).
    If \Cref{alg:meta-algorithm} resets \(\Icurrent\) on day \(t_{i+1}\), it means that \(I_i\) has error rate \(\le \frac{\alpha T}{t_{i+1} - 1}\) over the window \([1, t_{i+1} - 1]\) (\Cref{alg:meta-algorithm}, Line \ref{algline:reset-condition}).  By the uniform convergence property (\ref{eq:conservative-uc-property}), 
    \begin{align*}
        M^{\widehat{S}}_i &= \error_{\widehat{S}}(I_i, ~[t_i, t_{i + 1} - 1]) \\
        &\le \error_{\widehat{S}}(I_i, ~[1, T]) + \varepsilon_{[t_i, t_{i + 1} - 1]} \\
        &\le \error_{\widehat{S}}(I_i, [1, t_{i+1} - 1]) + \varepsilon_{[1, t_{i+1} - 1]} + \varepsilon_{[t_i, t_{i + 1} - 1]} \\
        &\le \frac{\alpha T}{t_{i + 1} - 1} + C \sqrt{\frac{\log T}{t_{i + 1} - 1}} + C \sqrt{\frac{\log T}{t_{i + 1} - t_i}} \\
        &\le \frac{\alpha T}{t_{i + 1} - 1} + 2C \sqrt{\frac{\log T}{t_{i + 1} - t_i}}.
    \end{align*}
    We bound \(M^{\widehat{S}}\) in terms of the \(M^{\widehat{S}}_i\)s.  Since \(R(t) = \frac{\alpha T}{t} \ge 1\) for \(t \le \alpha T\), and the initial interval has coverage 0, we have that \(t_2 = \alpha T + 1\) and \(M_1 = 1\).
    \begin{align}
        M^{\widehat{S}} &= \sum_{i = 1}^K (t_{i + 1} - t_i) M^{\widehat{S}}_i \nonumber \\
        &= \alpha T + \sum_{i = 2}^K (t_{i + 1} - t_i) M^{\widehat{S}}_i \nonumber\\
        &\le \alpha T + \sum_{i = 2}^K (t_{i + 1} - t_i) \left(\frac{\alpha T}{t_{i + 1} - 1} + 2C \sqrt{\frac{\log T}{t_{i + 1} - t_i}}\right) \nonumber\\
        &\le \alpha T + \sum_{i = 2}^K \left[(t_{i + 1} - t_i)\frac{\alpha T}{t_{i + 1} - 1} + 2C \sqrt{\log T }\sqrt{(t_{i + 1} - t_i) }\right] \nonumber\\
        &\le \alpha T + 2C\sqrt{\log T}\sqrt{K-1} \sqrt{T}  + \alpha T \sum_{i = 2}^K \left[\frac{t_{i + 1} - 1}{t_{i + 1} - 1} - \frac{t_i - 1}{t_{i + 1}-1} \right] &\text{by concavity of \(\sqrt{\cdot~}\)} \nonumber\\
        &\le \alpha T + 2C \sqrt{(K-1)T \log T} + (K - 1) \alpha T - \alpha T \sum_{i = 2}^K \frac{t_i - 1}{t_{i + 1}-1}. \label{eq:term-to-minimize}
    \end{align}
    This upper bound on mistakes is maximized when \(\sum_{i = 2}^K \frac{t_i - 1}{t_{i + 1}-1}\) is minimized.  Recall that \(t_2 = \alpha T + 1\) and \(t_{K+1} = T + 1\).  Assume without loss of generality that \(\alpha T\) is an integer.   So we have  
    \begin{align*}
        \prod_{i =2}^{K} \frac{t_{i} - 1}{t_{i + 1} - 1}  &= \frac{t_{2} - 1}{t_{K + 1} - 1} = \alpha \\
        \sum_{i = 2}^K \ln \left(\frac{t_{i} - 1}{t_{i + 1} - 1} \right) &= \ln \alpha .
    \end{align*}
    For convenience of arithmetic, define \(q_i = (t_{i} - 1)/(t_{i + 1} - 1)\) where every \(q_i > 0\), so we have \(\sum_{i = 2}^K \ln q_i = \ln \alpha\).
    By concavity of \(\ln\) we have
    \begin{align*}
        \frac{1}{K - 1} \sum_{i = 2}^K \ln q_i &\le \ln \left(\frac{1}{K - 1} \sum_{i = 2}^K q_i \right) \\
        \frac{\ln \alpha}{K-1} + \ln (K-1) &\le \left(\sum_{i = 2}^K q_i \right) \\
        (K-1) \alpha^{1/(K-1)} &\le \sum_{i = 2}^K q_i = \sum_{i = 2}^K \frac{t_{i} - 1}{t_{i + 1} - 1}.
    \end{align*}
    Substituting back into (\ref{eq:term-to-minimize}), we have 
    \begin{align*}
        M^{\widehat{S}} &\le \alpha T + \sqrt{(K-1)T\log T} + (K-1) \alpha T - (K-1)\alpha^{1/(K-1)} \alpha T \\
        &\le \alpha T + \sqrt{(K-1)T\log T} + (K-1)\left(1 - \alpha^{1/(K-1)} \right) \alpha T.
    \end{align*}
    Since \(\alpha \ge 0\), we have that \((K-1)\left(1 - \alpha^{1/(K-1)} \right) \le K-1\).  We can also bound this term by \(\ln (1/\alpha)\):
    \begin{align*}
        \frac{\ln \alpha}{K-1} &= -\frac{\ln (1/\alpha)}{K-1} \\
        \alpha^{1/(K-1)} &= e^{-\ln(1/\alpha) / (K-1)} \ge 1 - \frac{\ln(1/\alpha)}{K-1} \\
        \frac{\ln(1/\alpha)}{K-1} &\ge 1 - \alpha^{1/(K-1)} \\
        \ln(1/\alpha) &\ge (K-1) \left( 1 - \alpha^{1/(K-1)}\right).
    \end{align*}
    Putting this together we have 
    \begin{equation*}
        M^{\widehat{S}} \le \sqrt{(K- 1) T \log T} + \bigg(1 + \min \left\{K-1, ~\ln(1/\alpha) \right\}\bigg) \alpha T.
    \end{equation*}
    We can upper bound the number of phases of \Cref{alg:meta-algorithm} on any input as 
    \[K - 1 \le \log_{\frac{\mu - 1}{2}} (1/\minwidth) = O \left(\frac{\log(1/\minwidth)}{\log \mu} \right).\]
    So for every realization \(\widehat{S}\) of \(\mathbf{S}\) that satisfies the uniform convergence property (\ref{eq:conservative-uc-property}), we have 
    \begin{equation*}
        M^{\widehat{S}} = O\left( \sqrt{T\log T} \sqrt{\frac{\log(1/\minwidth)}{\log \mu}} + \min \left\{ \frac{\log(1/\minwidth)}{\log \mu}, ~ \log(1/\alpha) \right\} \alpha T \right).
    \end{equation*}
    We have that \(\mathbf{S}\) satisfies the uniform convergence property with probability \(\ge 1 - \frac{1}{T^2}\).  If \(\mathbf{S}\) does not satisfy the uniform convergence property, the number of mistakes is trivially upper bounded by \(T\).  So we have that the expected number of mistakes by \Cref{alg:meta-algorithm} (for \(\mu\), \(\minwidth\), \(R(t) = \frac{\alpha T}{t}\)) is 
    \begin{equation} 
    O\left( \sqrt{T\log T} \sqrt{\frac{\log(1/\minwidth)}{\log \mu}} + \min \left\{ \frac{\log(1/\minwidth)}{\log \mu}, ~ \log(1/\alpha) \right\} \alpha T \right). \label{eq:bad-sampling-error}
    \end{equation}
    When the number of phases \(\frac{\log(1/\minwidth)}{\log \mu}\) is large, the first ``sampling error" term will dominate.  In this setting, the original bound in \Cref{thm:meta-alg-exchangeable} which does not do a phase-by-phase analysis gives a tighter bound of 
    \begin{align}
        M^\mathbf{S} &\le T \left( \frac{1}{T} \sum_{t = 1}^T R(t - 1) + O\left(\sqrt{\log T / T}\right)\right) \nonumber \\
        &\le \sum_{t = 1}^T \max \left\{1, \frac{\alpha T}{t} \right\} + O \left( \sqrt{T \log T} \right) \nonumber \\
        &\le \alpha T + \alpha T \sum_{t = \alpha T + 1}^T \frac{1}{t} + O \left( \sqrt{T \log T} \right) \nonumber \\
        &\le O\left(\sqrt{T \log T} + \log (1/\alpha)\alpha T \right).  \label{eq:original-bound}
    \end{align}
    (\ref{eq:bad-sampling-error}) and (\ref{eq:original-bound}) together bound the expected number of mistakes by 
    \[O \left( \min \left\{\sqrt{T\log T} \sqrt{\frac{\log(1/\minwidth)}{\log \mu}} +  \frac{\log(1/\minwidth)}{\log \mu} \alpha T, \quad \sqrt{T \log T} + \log (1/\alpha)\alpha T \right\} \right).\]

    The volume guarantee follows directly from \Cref{cor:optimal-algorithm-for-arbitrary-order}.
\end{proof}

\section{Lower Bound for Arbitrary Order Sequences}


We provide a lower bound that shows that our algorithm for arbitrary order sequences is optimal.  That is, for any given choice of miscoverage rate \(\alpha\), scale lower bound \(\minwidth\), and multiplicative volume approximation \(\mu\), we prove that any algorithm that achieves volume \(\le \mu \min \{\opt_S(\alpha), ~\minwidth \}\) on every arbitrary order sequence \(S\) of length \(T\) must make a multiplicative factor more mistakes than the target \(\alpha T\).  This mistake bound shows that our upper bound of running \Cref{alg:meta-algorithm} with \(R(t) = \frac{\alpha T}{t}\) is near-optimal in all parameters.  In particular, we note that \Cref{alg:meta-algorithm} actually plays intervals with maximum volume bounded by \( \mu \min \{\opt_S(\alpha), ~\minwidth \}\), which is even stronger than playing intervals with average volume bounded by \( \mu \min \{\opt_S(\alpha), ~\minwidth \}\).  Among algorithms that have this property, we show that the expected number of mistakes that \Cref{alg:meta-algorithm} makes:
\[O \left( \min \left\{ \frac{\log (1/\minwidth)}{\log(\mu)} \cdot \alpha T , ~T \right\}\right),\]
is optimal (see \Cref{cor:optimal-algorithm-for-arbitrary-order} for the upper bound).  

For algorithms that play intervals that have average volume bounded by \( \mu \min \{\opt_S(\alpha), ~\minwidth \}\), but do not necessarily maximum volume bounded by this, we show that they must make 
\[\Omega \left( \left\{ \frac{\log (1/\minwidth)}{\log(\mu) + \log(1/\alpha)} \cdot \alpha T , ~T \right\} \right)\]
mistakes in expectation.  This differs from the earlier bound by a subpolynomial factor in \(\alpha\), and thus \Cref{alg:meta-algorithm} is near-optimal even among algorithms that meet this weaker condition.  

The lower bound construction is directly inspired by a worst-case instance for \Cref{alg:meta-algorithm} with \(R(t) = \frac{\alpha T}{t}\).  In the analysis (\Cref{thm:meta-alg-arbitrary-order}), we bound the number of times that \Cref{alg:meta-algorithm} resets \(\Icurrent\).  We then say between two subsequent resets of \(\Icurrent\), \Cref{alg:meta-algorithm} plays a fixed interval that has accrued less than the target error.  So that interval must make at most \(\alpha T\) mistakes over the time that it is feasible, and therefore at most \(\alpha T\) mistakes while it is being played.  In the lower bound, we present the algorithm with a series of phases of increasing ``scales," each lasting \(\alpha T\) days.  See \Cref{fig:arbitrary-lower-bound} for an illustration.  In each phase, the algorithm does not yet know if this phase contains points that are captured by \(\opt\), or this is an ``outlier" phase.  If the algorithm chooses to capture a large fraction of the points, it will accrue high volume, which leads to high multiplicative ratio in the case that this was an outlier phase.  If the algorithm chooses not to capture a large fraction of points, it will accrue a large number of mistakes.  

\begin{figure}[!t]
    \centering
    \begin{subfigure}{\textwidth}
        \centering
        \begin{tikzpicture}
            \fill[blue, opacity=0.6666] (0.25, 0) -- (0.25, 3*0.0877) -- (2.25, 3*0.0877) -- (2.25, 0);
            \fill[blue, opacity=0.4444] (2.25, 0) -- (2.25, 3*0.1316) -- (4.25, 3*0.1316) -- (4.25, 0);
            \fill[blue, opacity=0.2962] (4.25, 0) -- (4.25, 3*0.1975) -- (6.25, 3*0.1975) -- (6.25, 0);
            \fill[blue, opacity=0.1975] (6.25, 0) -- (6.25, 3*0.2962) -- (8.25, 3*0.2962) -- (8.25, 0);
            \fill[blue, opacity=0.1316] (8.25, 0) -- (8.25, 3*0.4444) -- (10.25, 3*0.4444) -- (10.25, 0);
            \fill[blue, opacity=0.0877] (10.25, 0) -- (10.25, 3*0.6666) -- (12.25, 3*0.6666) -- (12.25, 0);
            \fill[blue, opacity=0.0585] (12.25, 0) -- (12.25, 3*1) -- (14.25, 3*1) -- (14.25, 0);
        
            \draw[->, line width=1pt] (0, 0) -- (15, 0);
            \draw[dashed, gray] (0, 3) -- (15, 3);
            \draw[dashed, gray] (0, 3*0.4444) -- (15, 3*0.4444);
            \draw[->, line width=1pt] (0, 0) -- (0, 3.5);
            \node[anchor=north] (t=1) at (0.5, 0) {$1$};
            \node[anchor=north] (t=T) at (14, 0) {$T$};
            \node[anchor=west] (time) at (15, 0) {$t$};
            \node[anchor=east] (y=0) at (0, 0) {$0$};
            \node[anchor=east] (y=1) at (0, 3) {$1$};
            \node[anchor=east] (a) at (0, 3*0.6666) {$\varepsilon$};
            \node[anchor=east] (b) at (0, 3*0.4444) {$\varepsilon^2$};
            \node[anchor=east] (b) at (0, 3*0.2962) {$\vdots$};
            \node[anchor=south] (y) at (0, 3.5) {$y_t$};
    
            \node[anchor=north] (a) at (2.5, 0) {$\alpha T$};
            \node[anchor=north] (b) at (4.5, 0) {$2\alpha T \dots$};

            \draw[dashed, red] (0, 3*0.6666) -- (15, 3*0.6666);
            \draw[<->, line width=1pt, red] (14.5, 0) -- (14.5, 3*0.6666);
            \node[anchor=south, red] (c) at (14.5, 3*0.6666) {$\textsc{Opt}(\alpha)$};
        \end{tikzpicture}
        \caption{If an algorithm \(\alg\) has low expected coverage (\(\le \frac{1}{2}\)) on every phase of this sequence, it must make \(\Omega(T)\) mistakes over the whole sequence.  }
        \label{subfig:arbitrary-full-sequence}
    \end{subfigure}

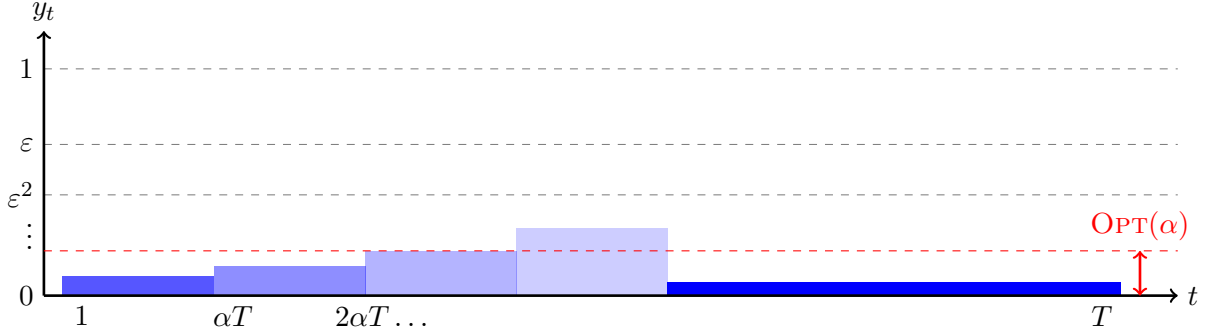
    \begin{subfigure}{\textwidth}
        \centering
        \begin{tikzpicture}
            \fill[blue, opacity=0.6666] (0.25, 0) -- (0.25, 3*0.0877) -- (2.25, 3*0.0877) -- (2.25, 0);
            \fill[blue, opacity=0.4444] (2.25, 0) -- (2.25, 3*0.1316) -- (4.25, 3*0.1316) -- (4.25, 0);
            \fill[blue, opacity=0.2962] (4.25, 0) -- (4.25, 3*0.1975) -- (6.25, 3*0.1975) -- (6.25, 0);
            \fill[blue, opacity=0.1975] (6.25, 0) -- (6.25, 3*0.2962) -- (8.25, 3*0.2962) -- (8.25, 0);
            \fill[blue] (8.25, 0) -- (8.25, 3*0.0585) -- (14.25, 3*0.0585) -- (14.25, 0);
        
            \draw[->, line width=1pt] (0, 0) -- (15, 0);
            \draw[dashed, gray] (0, 3) -- (15, 3);
            \draw[dashed, gray] (0, 3*0.6666) -- (15, 3*0.6666);
            \draw[dashed, gray] (0, 3*0.4444) -- (15, 3*0.4444);
            \draw[->, line width=1pt] (0, 0) -- (0, 3.5);
            \node[anchor=north] (t=1) at (0.5, 0) {$1$};
            \node[anchor=north] (t=T) at (14, 0) {$T$};
            \node[anchor=west] (time) at (15, 0) {$t$};
            \node[anchor=east] (y=0) at (0, 0) {$0$};
            \node[anchor=east] (y=1) at (0, 3) {$1$};
            \node[anchor=east] (a) at (0, 3*0.6666) {$\varepsilon$};
            \node[anchor=east] (b) at (0, 3*0.4444) {$\varepsilon^2$};
            \node[anchor=east] (b) at (0, 3*0.2962) {$\vdots$};
            \node[anchor=south] (y) at (0, 3.5) {$y_t$};
    
            \node[anchor=north] (a) at (2.5, 0) {$\alpha T$};
            \node[anchor=north] (b) at (4.5, 0) {$2\alpha T \dots$};
    
            \draw[dashed, red] (0, 3*0.1975) -- (15, 3*0.1975);
            \draw[<->, line width=1pt, red] (14.5, 0) -- (14.5, 3*0.1975);
            \node[anchor=south, red] (c) at (14.5, 3*0.1975) {$\textsc{Opt}(\alpha)$};
        \end{tikzpicture}
        
        \caption{
        If an algorithm \(\alg\) has high expected coverage (\(\ge \frac{1}{2}\)) on some phase \(i\) of sequence (a), say phase 4, then it also has high expected coverage on phase 4 of this alternate sequence, as \(\alg\) does not yet know which case it is in.  
        Therefore, \(\alg\) must incur a high multiplicative volume approximation on this alternate sequence, as the scale of phase 4 is multiplicatively larger than the scales of the other phases.
        }
        \label{subfig:arbitary-alternate-sequence}
    \end{subfigure}
    
    \caption{We illustrate the lower bound construction in \Cref{thm:arbitrary-order-lower-bound}. The horizontal axis is the day \(t\) of the sequence and the vertical axis is the value of \(y_t\).  \(y_t\) is drawn according to a uniform distribution, and we illustrate this as a box over the support of \(y_t\) with depth of color corresponding to the density of the distribution.  We divide the input sequence into ``phases" of length \(\alpha T\).  We then construct an input sequence \(S\) that has the scale of \(y_t\) increasing multiplicatively from one phase to the next (\Cref{subfig:arbitrary-full-sequence}).  We also consider a set of alternate sequences \(S_i\), where \(S_i\) is the same as \(S\) through phase \(i\), and then drops to a very small scale for the rest of the sequence (\Cref{subfig:arbitary-alternate-sequence}).  On phase \(i\), the algorithm does not know if it is in case (a), where the scale will continue to grow, or case (b), where the sequence will drop off.      }
    \label{fig:arbitrary-lower-bound}
\end{figure}


\begin{theorem}[Lower Bound for Arbitrary Sequences]\label{thm:arbitrary-order-lower-bound}
    \addcontentsline{toc}{subsection}{\emph{\Cref{thm:arbitrary-order-lower-bound}:} Lower Bound for Arbitrary Sequences}
    For any scale lower bound \(\minwidth > 0\), miscoverage rate \(\alpha > 0\), and time horizon \(T\), 
    such that for any, potentially randomized, algorithm \(\alg\) that outputs confidence sets over \([0,1]\) that are not necessarily intervals we have:
    \begin{enumerate}
        \item If \(\alg\) plays sets with expected \emph{average} volume 
        \(\le \muavg \max \{ \opt_{S_i}(\alpha), ~ \minwidth \}\)
        on every sequence \(S\), for some value \(\muavg>0\), then \(\alg\) must make 
        \[\Omega \left( \min \left\{ \frac{\log (1/\minwidth)}{\log(\muavg) + \log(1/\alpha)} \cdot \alpha T , ~T \right\}\right) \]
        mistakes in expectation on some sequence \(S'\). 
        This is \(\Omega \left( \min \left\{ \frac{\log(1/\minwidth)}{\log(\muavg)} \alpha^{1 + \varepsilon'} T , ~T \right\} \right)\) for any \(\varepsilon' > 0\).

        \item If \(\alg\) plays sets with expected \emph{maximum} volume 
        \(\le \mumax \max \left\{ \opt_{S_i} (\alpha), ~ \minwidth \right\} \)
        on every sequence \(S\), for some value \(\mumax>0\), then \(\alg\) must make 
        \[\Omega \left( \min \left\{ \frac{\log (1/\minwidth)}{\log(\mumax)} \cdot \alpha T , ~T \right\}\right) \]
        mistakes in expectation on some sequence \(S'\).
    \end{enumerate}

    Furthermore, this lower bound holds against algorithms that only have to be competitive on sequences where the input on each day is drawn from a symmetric unimodal distribution with median \(1/2\).
\end{theorem}

\begin{proof}
    Assume without loss of generality that \(1/\alpha\) is an integer, and \(T\) is a multiple of \(1/\alpha\).  Divide the input sequence into \(1 / \alpha\) ``phases" of length \(\alpha T\).  Let \(K \in \{1, \dots, 1/ \alpha \}\) and \(\varepsilon > 0\) be parameters to be chosen later. 
    For each \(1 \le i \le K\) we generate a sequence \(S^{(K)}_i\) as follows: 
    \begin{itemize}
        \item For phases \(j \le i\), all days in phase \(j\) are drawn i.i.d. from \(\unif \left[0, \varepsilon^{K - j} \right]\).
        \item For phases \(j > i\), all days in phase \(j\) are drawn i.i.d. from \(\unif \left[0, \varepsilon^K \right]\).
    \end{itemize}
    
    Fix any potentially randomized algorithm \(\alg\).  Consider what \(\alg\) does on input sequence \(S^{(K)}_K\).
    We say that \(\alg\) ``hits" phase \(j\) if it captures \(\ge \frac{\alpha T}{2}\) points in phase \(j\), in expectation.  If \(\alg\) captures \(< \frac{\alpha T}{2}\) points in phase \(j\) in expectation, we say it ``misses."  Either: 
    \begin{enumerate}[(a)]
        \item \(\alg\) misses every phase \(1 \le j \le K\).  Then \(\alg\) makes 
        \(\ge K \frac{\alpha T}{2}\)
        mistakes in expectation. 
        \item \(\alg\) hits some phase \(1 \le j \le K\).  Since phase \(j\) is drawn i.i.d.\ \(\unif \left[0, \varepsilon^{K - j} \right]\), the expected average volume that \(\alg\) plays on phase \(j\) must be \(\ge \frac{\varepsilon^{K - j}}{2}\).  

        Now consider what \(\alg\) would do on sequence \(S^{(K)}_j\).  Since \(S^{(K)}_j\) and \(S^{(K)}_K\) are identical through phase \(j\), this means that \(\alg\) must have expected average volume \(\ge \frac{\varepsilon^{K - j}}{2}\) on phase \(j\) of \(S^{(K)}_j\) as well.  Thus, on \(S^{(K)}_j\), \(\alg\) must play average volume 
        \(\ge \frac{\alpha \varepsilon^{K - j}}{2}\)
        and maximum volume 
        \(\ge \frac{\varepsilon^{K - j}}{2}\)
        in expectation.  

        For \(S^{(K)}_j\), we know that the interval \([0, \varepsilon^{K - j + 1}]\) achieves coverage 1 on every phase other than phase \(j\). Thus, \(\opt_{S^{(K)}_j}(\alpha) \le \varepsilon^{K - j + 1}\).  This means that on \(S^{(K)}_j\), \(\alg\) plays intervals of average volume 
        \[\ge \frac{\alpha}{2 \varepsilon} \opt_{S^{(K)}_j}(\alpha) \]
        in expectation, and intervals of maximum volume 
        \[\ge \frac{1}{2 \varepsilon} \opt_{S^{(K)}_j}(\alpha)\] in expectation. 
\end{enumerate}

We will set \(\varepsilon, K\) such that \(\minwidth = \frac{1}{2}\varepsilon^{K}\).  If \(\alg\) plays intervals with expected average volume \(\le \muavg \max \{\opt_{S_j^{(K)}}(\alpha), ~\minwidth \}\) on \(S_j^{(K)}\) we have that, either \(\alg\) makes \(\ge K \frac{\alpha T}{2}\) mistakes in expectation, or
\begin{align}
    \frac{\alpha}{2 \varepsilon} \opt_{S_j^{(K)}}(\alpha) &\le \muavg \max \{\opt_{S_j^{(K)}}(\alpha), ~\minwidth \} \nonumber \\
    \min \{ \frac{\alpha}{2 \varepsilon}, \frac{\alpha}{2 \varepsilon^{K+2}} \opt_{S_j^{(K)}}(\alpha) \} &\le \muavg \nonumber \\
    \frac{\alpha}{2 (\minwidth/2)^{1/K}} = \frac{\alpha}{2 \varepsilon} &\le \muavg , \label{eq:muavg-condition}
\end{align}
where the last line follows because \(\opt_{S_j^{(K)}}(\alpha) \ge \varepsilon^{K+1}\).
Similarly, if \(\alg\) plays intervals with expected maximum volume \(\le \mumax \max \{\opt_{S_j^{(K)}}(\alpha), ~\minwidth \}\) on \(S_j^{(K)}\) we have that, either \(\alg\) makes \(\ge K \frac{\alpha T}{2}\) mistakes in expectation, or
\begin{align}
    \frac{1}{2 (\minwidth/2)^{1/K}} = \frac{1}{2 \varepsilon} &\le \mumax . \label{eq:mumax-condition}
\end{align}
Setting \(K = \min \{ \left \lfloor \frac{\log (2/\minwidth)}{\log(1/\alpha) + \log(2\muavg)} \right \rfloor, 1/\alpha\}\) ensures that (\ref{eq:muavg-condition}) is satisfied.  Thus if \(\alg\) has expected average volume bounded as \(\le \muavg \max \{\opt_{S^{(K)}_j}(\alpha), ~\minwidth\}\) for all \(S^{(K)}_j \in \sequences\), then \(\alg\) must make 
\begin{align*}
    &\ge \min \left\{ \left(\frac{\log (2 / \minwidth)}{\log (1/\alpha) + \log (2\muavg)} - 1 \right)  \alpha T  , ~\frac{T}{2} \right\} \\
    &= \Omega \left( \min \left\{ \frac{\log(1/\minwidth)}{\log(\muavg) + \log(1/\alpha)} \alpha T , ~T \right\} \right)\\
    &= \Omega \left( \min \left\{ \frac{\log(1/\minwidth)}{\log(\muavg)} \alpha^{1 + \varepsilon'} T , ~T \right\} \right) & \text{for any } \varepsilon' > 0,
\end{align*}
mistakes in expectation on \(S^{(K)}_K\).  
Similarly, Setting \(K = \min \{ \left\lfloor \frac{\log (2/\minwidth)}{\log (2 \mumax)} \right\rfloor, ~1/\alpha \}\) ensures that (\ref{eq:mumax-condition}) is satisfied and \(K \le 1/\alpha\).  Thus if \(\alg\) has expected maximum volume bounded as \(\le \mumax \max \{\opt_{S_j}(\alpha), ~\minwidth\}\) for all \(S^{(K)}_j \in \sequences^{(K)}\), then \(\alg\) must make 
\begin{align*}
    &\ge \min \left\{ \left( \frac{\log (2/\minwidth)}{\log (2 \mumax)} - 1 \right) \alpha T, ~\frac{T}{2} \right\} \\
    &= \Omega \left(\min \left\{ \frac{\log (1/\minwidth)}{\log (\mumax)} \alpha T, ~T \right\} \right) 
\end{align*}
mistakes in expectation on \(S^{(K)}_K\).  

Finally, we note that while this argument considered the performance of \(\alg\) on intervals with one end at \(0\) for ease of notation, we could have instead considered a symmetric version of the sequences, where we map each phase that is drawn
\[\unif [0, b] \rightarrow \unif \left[\frac{1}{2} - \frac{1}{2} b , ~\frac{1}{2} + \frac{1}{2}b \right],\]
(see \Cref{fig:symmetric-arbitrary-lower-bound}).
This leads to the same lower bounds on volume, while ensuring that the input distribution on each day is unimodal and symmetric around \(1/2\).  This gives a lower bound against algorithms that are competitive on such sequences.
\end{proof}

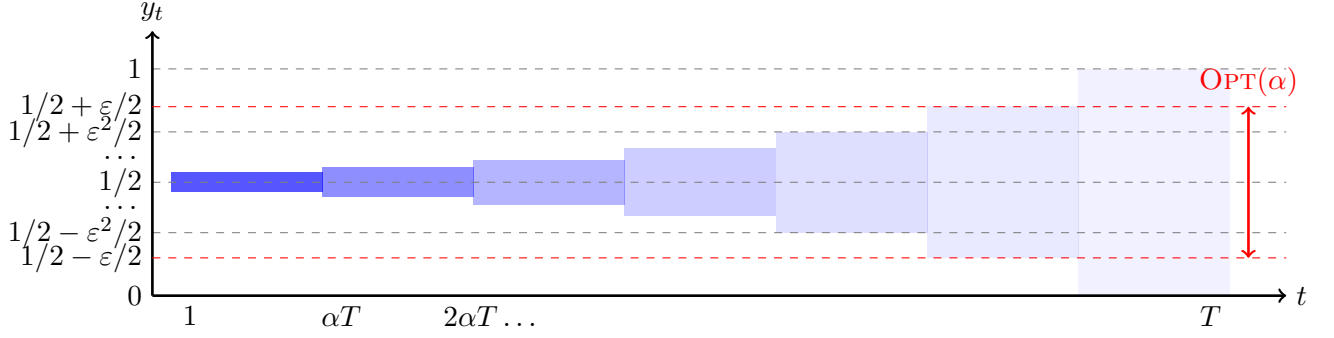
\begin{figure}[!t]
    \centering
    \begin{tikzpicture}
        \fill[blue, opacity=0.6666] (0.25, 3/2 - 3/2*0.0877) -- (0.25, 3/2 + 3/2*0.0877) -- (2.25, 3/2 + 3/2*0.0877) -- (2.25, 3/2 - 3/2*0.0877);
        \fill[blue, opacity=0.4444] (2.25, 3/2 - 3/2*0.1316) -- (2.25, 3/2 + 3/2*0.1316) -- (4.25, 3/2 + 3/2*0.1316) -- (4.25, 3/2 - 3/2*0.1316);
        \fill[blue, opacity=0.2962] (4.25, 3/2 - 3/2*0.1975) -- (4.25, 3/2 + 3/2*0.1975) -- (6.25, 3/2 + 3/2*0.1975) -- (6.25, 3/2 - 3/2*0.1975);
        \fill[blue, opacity=0.1975] (6.25, 3/2 - 3/2*0.2962) -- (6.25, 3/2 + 3/2*0.2962) -- (8.25, 3/2 + 3/2*0.2962) -- (8.25, 3/2 - 3/2*0.2962);
        \fill[blue, opacity=0.1316] (8.25, 3/2 - 3/2*0.4444) -- (8.25, 3/2 + 3/2*0.4444) -- (10.25, 3/2 + 3/2*0.4444) -- (10.25, 3/2 - 3/2*0.4444);
        \fill[blue, opacity=0.0877] (10.25, 3/2 - 3/2*0.6666) -- (10.25, 3/2 + 3/2*0.6666) -- (12.25, 3/2 + 3/2*0.6666) -- (12.25, 3/2 - 3/2*0.6666);
        \fill[blue, opacity=0.0585] (12.25, 0) -- (12.25, 3*1) -- (14.25, 3*1) -- (14.25, 0);
    
        \draw[->, line width=1pt] (0, 0) -- (15, 0);
        \draw[dashed, gray] (0, 3) -- (15, 3);
        \draw[dashed, gray] (0, 3/2) -- (15, 3/2);
        \draw[dashed, gray] (0, 3/2 + 3/2*0.4444) -- (15, 3/2 + 3/2*0.4444);
        \draw[dashed, gray] (0, 3/2- 3/2*0.4444) -- (15, 3/2-3/2*0.4444);
        \draw[->, line width=1pt] (0, 0) -- (0, 3.5);
        \node[anchor=north] (t=1) at (0.5, 0) {$1$};
        \node[anchor=north] (t=T) at (14, 0) {$T$};
        \node[anchor=west] (time) at (15, 0) {$t$};
        \node[anchor=east] (y=0) at (0, 0) {$0$};
        \node[anchor=east] (y=1) at (0, 3) {$1$};
        \node[anchor=east] (y=1/2) at (0, 3/2) {$1/2$};
        \node[anchor=east] (a) at (0, 3/2 + 3/2*0.6666) {$1/2 + \varepsilon/2$};
        \node[anchor=east] (a) at (0, 3/2 - 3/2*0.6666) {$1/2 - \varepsilon/2$};
        \node[anchor=east] (b) at (0, 3/2 + 3/2*0.4444) {$1/2 + \varepsilon^2/2$};
        \node[anchor=east] (b) at (0, 3/2 - 3/2*0.4444) {$1/2 - \varepsilon^2/2$};
        \node[anchor=east] (b) at (0, 3/2 + 3/4*0.4444) {$\dots$};
        \node[anchor=east] (b) at (0, 3/2 - 3/4*0.4444) {$\dots$};
        \node[anchor=south] (y) at (0, 3.5) {$y_t$};

        \node[anchor=north] (a) at (2.5, 0) {$\alpha T$};
        \node[anchor=north] (b) at (4.5, 0) {$2\alpha T \dots$};

        \draw[dashed, red] (0, 3/2 - 3/2*0.6666) -- (15, 3/2 - 3/2*0.6666);
        \draw[dashed, red] (0, 3/2 + 3/2*0.6666) -- (15, 3/2 + 3/2*0.6666);
        \draw[<->, line width=1pt, red] (14.5, 3/2 - 3/2*0.6666) -- (14.5, 3/2 + 3/2*0.6666);
        \node[anchor=south, red] (c) at (14.5, 3/2 + 3/2*0.6666) {$\textsc{Opt}(\alpha)$};
    \end{tikzpicture}
    \caption{Symmetric version of the lower bound construction in \Cref{fig:arbitrary-lower-bound}}
    \label{fig:symmetric-arbitrary-lower-bound}
\end{figure}

\section{Lower Bound Against ``Best of Both Worlds" Algorithms}

We provide a lower bound that says that it is impossible to for a single algorithm to achieve the best possible guarantee both for arbitrary sequences and exchangeable sequences.  Recall that for an arbitrary sequence \(S\), \Cref{cor:optimal-algorithm-for-arbitrary-order} shows that for choices of the multiplicative volume approximation \(\mu\), scale lower bound \(\minwidth\), and target miscoverage \(\alpha\), \Cref{alg:meta-algorithm} with \(R(t) = \frac{\alpha T}{t}\) plays intervals of average volume \(\le \mu \max \{\opt_S(\alpha), ~\minwidth \}\), and makes \(O( \log(1/\minwidth)/\log(\mu) \alpha T)\) mistakes.  \Cref{thm:arbitrary-order-lower-bound} tells us that this is optimal up to subpolynomial factors in \(\alpha\).  

For exchangeable sequences however, we can do better.  \Cref{cor:optimal-algorithm-for-exchangeable-sequences} tells us that \Cref{alg:meta-algorithm} with \(R(t) = (1 - \alpha) + O(\sqrt{\log T/t})\) plays intervals of average volume \(\lesssim \opt_S(\alpha)\) and has expected coverage \(\ge (1 - \alpha) - O(\sqrt{\log T/t})\) (that is, \(\le (1 + O(\sqrt{\log T/t})) \alpha T\) mistakes in expectation).  

This leads to the natural question: is there a single algorithm that can achieve both guarantees?  It turns out that these two bounds are indeed at odds.  That is, achieving an algorithm that achieves the optimal tradeoff for arbitrary sequences must make a large number of mistakes on some exchangeable sequence, and an algorithm that achieves the optimal number of mistakes for exchangeable sequences must have a large multiplicative volume approximation.  

The following lower bound gives a tradeoff between the two cases.  \Cref{thm:best-of-both-worlds-algorithm} tells us that \Cref{alg:meta-algorithm} with \(R(t) = \frac{\alpha T}{t}\) actually achieves this tradeoff, and thus it is optimal in the sense that it gets as close as possible to a ``best of both worlds" guarantee.  That is, for a fixed choice of \(\mu\), \(\minwidth\), and \(\alpha\), subject to playing intervals of average volume \(\le \mu \max \{ \opt_S(\alpha), \minwidth \}\) on every arbitrary sequence \(S\), \Cref{alg:meta-algorithm} with \(R(t) = \frac{\alpha T}{t}\) makes the lowest possible number of mistakes on both arbitrary and exchangeable sequences.  We remark, however, that the lower bound does not say anything about whether the multiplicative factor of \(\mu\) that \Cref{alg:meta-algorithm} incurs in the exchangeable case is necessary.  

The construction of this lower bound, like the lower bound for arbitrary sequences, is inspired by observing the performance of \Cref{alg:meta-algorithm} with \(R(t) = \frac{\alpha T}{t}\) on i.i.d.\ sequences.  The optimal algorithm for exchangeable sequences will play intervals that have coverage very close to the target \(1 - \alpha\) from very early on.  \Cref{alg:meta-algorithm} with \(R(t) = \frac{\alpha T}{t}\) however, will be more conservative and play intervals that have substantially lower coverage, to hedge against the possibility the scale of the problem will shrink in the future, and that many of the points seen so far were actually ``outliers."  Intuitively, such a strategy should be optimal subject to a worst case guarantee.  

The lower bound proof essentially directly argues that this conservative strategy is optimal.  That is, we design a distribution \(\distribution\), and then argue that any algorithm \(\alg\) must fall into one of two categories.
\begin{enumerate}[(a)]
    \item On every day \(t\), \(\alg\) plays an interval that has expected miscoverage \(\ge \frac{1}{2} \cdot\frac{\alpha T}{t} \). Summing over \(t\), we see that this leads to \(\approx \ln(1/\alpha) \alpha T\) mistakes in expectation overall.

    \item On some day \(t\), \(\alg\) plays an interval with expected miscoverage \(\le \frac{1}{2} \cdot \frac{\alpha T}{t}\).  Then, we show that \(\alg\) has not been conservative, and there exists an alternate sequence on which \(\alg\) incurs a large volume approximation.
\end{enumerate}
A core part of the argument is in designing a distribution \(\distribution\) that achieves the desired tradeoff between the miscoverage of the algorithm and the multiplicative volume approximation \(\mu\).  We provide an illustration of the argument in \Cref{fig:iid-lower-bound}.

\begin{figure}[!t]
    \centering
    \begin{subfigure}{\textwidth}
        \centering 
        \begin{tikzpicture}
            \shade[bottom color = blue!66.66!white, top color = blue!44.44!white] (0.25, 3*0.0585) rectangle (14.25, 3*0.0877);
            \shade[bottom color = blue!44.44!white, top color=blue!29.62!white] (0.25, 3*0.0877) rectangle (14.25, 3*0.1316); 
            \shade[bottom color = blue!29.62!white, top color=blue!19.75!white] (0.25, 3*0.1316) rectangle (14.25, 3*0.1975);
            \shade[bottom color = blue!19.75!white, top color = blue!13.16!white] (0.25, 3*0.1975) rectangle (14.25, 3*0.2962);
            \shade[bottom color = blue!13.16!white, top color=blue!8.77!white] (0.25, 3*0.2962) rectangle (14.25, 3*0.4444);
            \shade[bottom color = blue!8.77!white, top color = blue!5.85!white] (0.25, 3*0.4444) rectangle (14.25, 3*0.6666);
            \shade[bottom color = blue!5.85!white, top color = white] (0.25, 3*0.6666) rectangle (14.25, 3*1);
        
            \draw[->, line width=1pt] (0, 0) -- (15, 0);
            \draw[dashed, gray] (0, 3) -- (15, 3);
            \draw[dashed, gray] (0, 3*0.4444) -- (15, 3*0.4444);
            \draw[->, line width=1pt] (0, 0) -- (0, 3.5);
            \node[anchor=north] (t=1) at (0.5, 0) {$1$};
            \node[anchor=north] (t=T) at (14, 0) {$T$};
            \node[anchor=west] (time) at (15, 0) {$t$};
            \node[anchor=east] (y=0) at (0, 0) {$0$};
            \node[anchor=east] (y=1) at (0, 3) {$1$};
            \node[anchor=east] (a) at (0, 3*0.6666) {$\varepsilon$};
            \node[anchor=east] (b) at (0, 3*0.4444) {$\varepsilon^2$};
            \node[anchor=east] (b) at (0, 3*0.2962) {$\vdots$};
            \node[anchor=south] (y) at (0, 3.5) {$y_t$};
    
            \draw[dashed, red] (0, 3*0.6666) -- (15, 3*0.6666);
            \draw[dashed, red] (0, 3*0.0585) -- (15, 3*0.0585);
            \draw[<->, line width=1pt, red] (14.5, 3*0.0585) -- (14.5, 3*0.6666);
            \node[anchor=south, red] (c) at (14.5, 3*0.6666) {$\textsc{Opt}(\alpha)$};
        \end{tikzpicture}
        \caption{In this sequence, every day's input is drawn i.i.d.\ from \(\distribution\).  If an algorithm \(\alg\) has miscoverage \(\ge \frac{1}{2} \cdot \frac{\alpha T}{t}\) on every day, it will make a large number (\(\ge \frac{1}{2} \ln (1/\alpha) \alpha T\)) of mistakes over the whole sequence.}
    \end{subfigure}

    
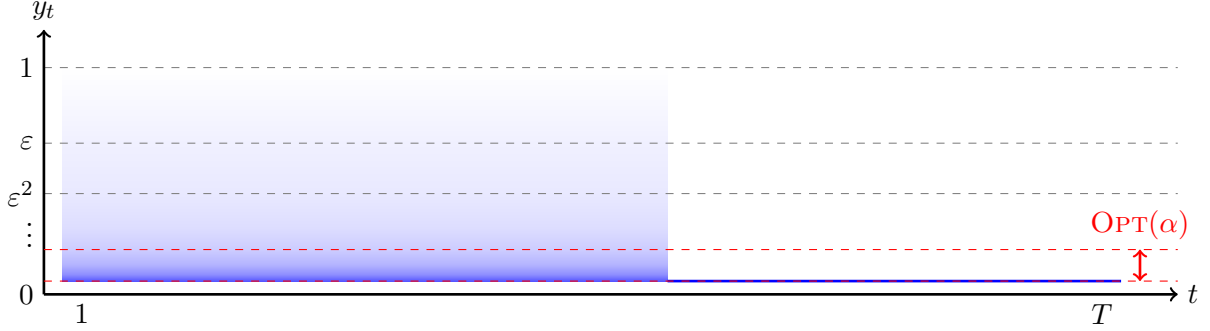
\begin{subfigure}{\textwidth}
        \centering 
        \begin{tikzpicture}
            \shade[bottom color = blue!66.66!white, top color = blue!44.44!white] (0.25, 3*0.0585) rectangle (8.25, 3*0.0877);
            \shade[bottom color = blue!44.44!white, top color=blue!29.62!white] (0.25, 3*0.0877) rectangle (8.25, 3*0.1316); 
            \shade[bottom color = blue!29.62!white, top color=blue!19.75!white] (0.25, 3*0.1316) rectangle (8.25, 3*0.1975);
            \shade[bottom color = blue!19.75!white, top color = blue!13.16!white] (0.25, 3*0.1975) rectangle (8.25, 3*0.2962);
            \shade[bottom color = blue!13.16!white, top color=blue!8.77!white] (0.25, 3*0.2962) rectangle (8.25, 3*0.4444);
            \shade[bottom color = blue!8.77!white, top color = blue!5.85!white] (0.25, 3*0.4444) rectangle (8.25, 3*0.6666);
            \shade[bottom color = blue!5.85!white, top color = white] (0.25, 3*0.6666) rectangle (8.25, 3*1);
    
            \draw[line width = 1pt, blue] (8.25, 3*0.0585) -- (14.25, 3*0.0585);
        
            \draw[->, line width=1pt] (0, 0) -- (15, 0);
            \draw[dashed, gray] (0, 3) -- (15, 3);
            \draw[dashed, gray] (0, 3*0.6666) -- (15, 3*0.6666);
            \draw[dashed, gray] (0, 3*0.4444) -- (15, 3*0.4444);
            \draw[->, line width=1pt] (0, 0) -- (0, 3.5);
            \node[anchor=north] (t=1) at (0.5, 0) {$1$};
            \node[anchor=north] (t=T) at (14, 0) {$T$};
            \node[anchor=west] (time) at (15, 0) {$t$};
            \node[anchor=east] (y=0) at (0, 0) {$0$};
            \node[anchor=east] (y=1) at (0, 3) {$1$};
            \node[anchor=east] (a) at (0, 3*0.6666) {$\varepsilon$};
            \node[anchor=east] (b) at (0, 3*0.4444) {$\varepsilon^2$};
            \node[anchor=east] (b) at (0, 3*0.2962) {$\vdots$};
            \node[anchor=south] (y) at (0, 3.5) {$y_t$};
    
            \draw[dashed, red] (0, 3*0.1975) -- (15, 3*0.1975);
            \draw[dashed, red] (0, 3*0.0585) -- (15, 3*0.0585);
            \draw[<->, line width=1pt, red] (14.5, 3*0.0585) -- (14.5, 3*0.1975);
            \node[anchor=south, red] (c) at (14.5, 3*0.1975) {$\textsc{Opt}(\alpha)$};
        \end{tikzpicture}
        \caption{In this sequence, the prefix is drawn i.i.d.\ from \(\distribution\), but later the sequence drops off to be deterministically one value.  If the algorithm \(\alg\) does not play conservatively on the prefix of the sequence, it risks accumulating average volume much larger than \(\opt\).  }
    \end{subfigure}
    
    \caption{We illustrate the lower bound construction in \Cref{thm:no-best-of-both-worlds}. The horizontal axis is the day \(t\) of the sequence and the vertical axis is the value of \(y_t\).  \(y_t\) is drawn according to a distribution \(\distribution\) that we illustrate with depth of color corresponding to the density of the distribution.  \(\distribution\) has the property that the minimum volume interval that achieves miscoverage \(\le \alpha e^i\) over \(\distribution\) is multiplicatively larger than the minimum volume interval that achieves miscoverage \(\le \alpha e^{i + 1}\). }
    \label{fig:iid-lower-bound}
\end{figure}

\begin{theorem}[No ``Best of Both Worlds" Algorithm]\label{thm:no-best-of-both-worlds}
\addcontentsline{toc}{subsection}{\emph{\Cref{thm:no-best-of-both-worlds}:} No ``Best of Both Worlds" Algorithm}
Fix a scale lower bound \(\minwidth > 0\) and target miscoverage rate \(\alpha > 0\).  Let \(\alg\) be any algorithm, deterministic or randomized. 
\begin{enumerate}
    \item If \(\alg\) plays sets with expected \emph{average} volume \(\le \muavg \max \left\{ \opt_{S}(\alpha), ~\minwidth  \right\}\) on every arbitrary order sequence \(S\), for some \(\muavg > 0 \), then \(\alg\) must make 
    \[\Omega \left( \min \left\{ \ln(1/\alpha), ~\frac{\ln (1/\minwidth)}{\ln \muavg + \ln(1/\alpha)} \right\} \alpha T \right)\]
    mistakes in expectation on some i.i.d.\ sequence. 

    \item If \(\alg\) plays sets with expected \emph{maximum} volume \(\le \mumax \max \left\{ \opt_{S}(\alpha), ~\minwidth  \right\}\) on every arbitrary order sequence \(S\), for some \(\muavg > 0 \), then \(\alg\) must make 
    \[\Omega \left( \min \left\{ \ln(1/\alpha), ~\frac{\ln (1/\minwidth)}{\ln \mumax} \right\} \alpha T \right)\]
    mistakes in expectation on some i.i.d.\ sequence. 
\end{enumerate}

Furthermore, this lower bound holds against algorithms that only have to be competitive on sequences where the input on each day is drawn from a symmetric unimodal distribution with median \(1/2\).
\end{theorem}

\begin{proof}
    Let \(0 < \varepsilon \le 1/2\) and \(0 \le K \le \ln(1/\alpha)\) be parameters that we will set later.  
    We define a distribution \(\distribution^{(K)}\) over \([0, 1]\), such that 
    \begin{align}
        \ProbOp_{x \sim \distribution^{(K)}} \left[x < \varepsilon^{K + 1} \right] &= 0, \nonumber\\
        \ProbOp_{x \sim \distribution^{(K)}}\left[ x \le \varepsilon^{1 + K - i} \right] &= 1 - \alpha e^{K - i} & \text{for } 0 \le i \le K, \label{eq:DK-definition}\\ 
        \ProbOp_{x \sim \distribution^{(K)}} \left[x \le \varepsilon + t( 1- \varepsilon) \right] &= 1 - \alpha (1 - t) &\text{for } 0 \le t \le 1, \nonumber
    \end{align}
    where \(e\) is the base of the natural logarithm. 
    Note that \(1 - \alpha e^{K}\) is a valid probability because \(\ln (\alpha e^{K}) = K - \ln (1/\alpha) \le 0 \) so \(\alpha e^{K} \le 1\).  

    We compute \(v^\star(c)\), which is the minimum volume of any set that achieves coverage \(c\) over \(\distribution^{(K)}\).  We take the c.d.f.\ of \(\distribution^{(K)}\),  
    \begin{equation*}
        F_{\distribution^{(K)}}(x) = \begin{cases}
            0 &\text{for } 0 \le x \le \varepsilon^{K+1}\\
            1 - \frac{\alpha}{e} x^{1/\ln \varepsilon} &\text{for } \varepsilon^{K + 1} \le x \le \varepsilon \\
            1 - \alpha \frac{1-x}{1-\varepsilon} &\text{for } \varepsilon \le x \le 1
        \end{cases} .
    \end{equation*}
    From the c.d.f. we can compute the p.d.f. of \(\distribution^{(K)}\).  For \(0 \le x < \varepsilon^{K + 1}\), \(\distribution^{(K)}\) has density 0.  At \(x = \varepsilon^{K+1}\), \(\distribution^{(K)}\) has a point mass of probability \(1 - \alpha e^K\).  For \(\varepsilon^{K + 1} < x \le 1\), we compute the density of \(\distribution^{(K)}\) by taking the derivative of the c.d.f. with respect to \(x\):
    \begin{equation*}
        f_{\distribution^{(K)}}(x) = \begin{cases}
           \frac{\alpha}{e \ln(1/\varepsilon)} x^{-\frac{1}{\ln (1/\varepsilon)} - 1} & \text{for } \varepsilon^{K + 1} < x < \varepsilon \\
            \frac{\alpha}{1-\varepsilon} & \text{for } \varepsilon < x \le 1
        \end{cases}.
    \end{equation*}
    This p.d.f.\ is non-increasing with \(x\) for \(\varepsilon^{K+1} \le x \le 1\).  At \(x = \varepsilon^{K+1}\) the p.d.f.\ is infinite.  The first piece is proportional to \(x\) to a constant negative power, and the second piece is constant.  So it suffices to verify that the density is non-increasing where the two pieces meet.  At \(x = \varepsilon\), the first piece would give density \(\alpha \frac{1/\varepsilon}{\ln(1/\varepsilon)}\).  Since \((\ln z) + 1 \le z\) for all \(z > 0\), 
    \begin{equation*}
        \frac{1}{\varepsilon} - 1\ge \ln(1/\varepsilon) \quad \Longrightarrow \quad \frac{1/\varepsilon}{\ln(1/\varepsilon)} \ge \frac{1}{1 - \varepsilon} \quad \Longrightarrow \quad \alpha \frac{1/\varepsilon}{\ln(1/\varepsilon)} \ge \frac{\alpha}{1 - \varepsilon}.
    \end{equation*}
    Since the density is non-increasing for \(x \ge \varepsilon^{K+1}\), and \(\distribution^{(K)}\) has no mass on \(x < \varepsilon^{K+1}\), for any coverage level \(0 \le c \le 1\), a minimum volume set achieving coverage \(c\) over \(\distribution^{(K)}\) is of the form \([\varepsilon^{K+1}, F_{\distribution^{(K)}}^{-1}(c)]\), so 
    \[v^\star(c) = F_{\distribution^{(K)}}^{-1}(c) - \varepsilon^{K+1},\]
    (where we define \(F^{-1}_{\distribution^{(K)}}(0) = \varepsilon^{K+1}\)).
    Since the derivative (p.d.f.) of \(F_{\distribution^{(K)}}\) is non-increasing for \(\varepsilon^{K+1} \le x \le 1\), \(F_{\distribution^{(K)}}(x)\) is concave over \(x \in [\varepsilon^{K+1}, 1]\).  So \(F^{-1}_{\distribution^{(K)}}(c)\) is convex over \(c \in [0, 1]\), and \(v^\star(c)\) is also convex.  

    The convexity of \(v^\star(c)\) along with Jensen's inequality imply the following. 
    \begin{fact}\label{fact:convexity-of-DK}
        Let \(\mathbf{I}\) be a (randomized) set of intervals over \([0, 1]\) that have average expected coverage \(\ge c\) over \(\distribution^{(K)}\).  Then, the average expected volume of the intervals \(\mathbf{I}\) is 
        \[\ge v^\star(c) = F_{\distribution^{(K)}}^{-1}(c) - \varepsilon^{K+1}.\]
    \end{fact}

    We use \(\distribution^{(K)}\) to design a lower bound instance.  
    Let \(S^{(K)}\) be the sequence of \(T\) i.i.d.\ draws from \(\distribution^{(K)}\).  Assume also that \(\minwidth \le \varepsilon^{K+1}\).  We will analyze the behavior of an algorithm \(\alg\) based on its expected performance on \(S^{(K)}\).  Fix an algorithm \(\alg\).  Let \(M\) be the expected number of mistakes that \(\alg\) makes on input \(S^{(K)}\).  Let \(M_t\) be the expected miscoverage of \(\alg\) on day \(t\) on input \(S^{(K)}\).  That is, \(M_t\) is the marginal probability that \(\alg\) does not capture a point on day \(t\) of \(S^{(K)}\), and \(M = M_1 + \dots + M_t.\)
    
    Let \(1 \le w \le T\) be a window length that we will set later, and let \(W_t\) denote the window (sequence of consecutive days) of length \(w\) that ends on day \(t\).  If \(\alg\) has average expected miscoverage \(\le \frac{1}{e} \frac{\alpha T}{t}\) over \(W_t\) on input \(S^{(K)}\), then we say that \(\alg\) ``hits" \(W_t\).  Otherwise, \(\alg\) has average expected miscoverage \(> \frac{1}{e} \frac{\alpha T}{t}\) over \(W_t\) on input \(S^{(K)}\), and we say that \(\alg\) ``misses" \(W_t\).

    We analyze by cases.  Either \(\alg\) hits some window \(W_t\) for \(T/e^{K -1} \le t \le T/e\), or \(\alg\) misses every window \(W_t\) for \(T/e^{K -1} \le t \le T/e\).
    \begin{enumerate}[(a)]
        \item \(\alg\) hits some window \(W_t\) for \(T/e^{K-1} \le t \le T/e\).  By definition, this means that \(\alg\) has average expected coverage \(\ge 1- \frac{1}{e} \frac{\alpha T}{t}\) over \(W_t\).  Fact \ref{fact:convexity-of-DK} tells us that we can bound the expected volume of the intervals that \(\alg\) plays over \(W_t\) using the inverse c.d.f.. We use the original definition of \(\distribution^{(K)}\) (\ref{eq:DK-definition}) for convenience.  For \(T/e^{K-1} \le t \le T/e\), expected coverage of \(1- \frac{1}{e} \frac{\alpha T}{t}\) corresponds to \(i = K+1 - \ln(T/t)\).  Thus, the expected average volume of \(\alg\) over \(W_t\) is 
        \begin{equation*}
            \ge \varepsilon^{1 + K - i} - \varepsilon^{K + 1} 
            \quad = \varepsilon^{\ln(T/t)} - \varepsilon^{K+1}.
        \end{equation*}

        Consider an alternate input sequence \(S_t^{(K)}\), that has the first \(t\) days drawn i.i.d.\ from \(\distribution^{(K)}\), and has future days deterministically set to \(\varepsilon^{K+1}\).  As \(T \rightarrow \infty\) (and the sampling error disappears, since \(t \ge T/e^{K+1}\) is also going to infinity), \(\opt_{S_t^{(K)}}(\alpha)\) will converge to the volume of the smallest interval that achieves coverage \(\frac{\alpha T}{t}\) over \(\distribution^{(K)}\).  For \(T/e^{K} \le t \le T/e\), we fall into the middle case of (\ref{eq:DK-definition}) and we have
        \begin{equation}
            \opt_{S_t^{(K)}}(\alpha) \rightarrow \varepsilon^{1 + \ln(T/t)} - \varepsilon^{K+1} \qquad \text{as } T \rightarrow \infty.\label{eq:opt-value}
        \end{equation}

        Let \(\volmax\) be the expected maximum volume of any interval played by \(\alg\) on \(S_t^{(K)}\), and let \(\volavg\) be the expected average volume played by \(\alg\) on \(S_t^{(K)}\).  By definition, we have that \[\volmax \le \mumax \max \{\opt_{S_t^{(K)}}(\alpha), \minwidth \}, \qquad \volavg \le \muavg \max \{\opt_{S_t^{(K)}}(\alpha), \minwidth \}.\]
        To lower bound the multiplicative ratio, we need to be in the case where \(\opt\) is larger than \(\minwidth\).  Using our assumption that \(\minwidth \le \varepsilon^{K+1}\), it suffices to show that 
        \begin{equation*}
            \varepsilon^{1 + \ln(T/t)} - \varepsilon^{K+1} \ge \varepsilon^{K+1} \quad \Longleftrightarrow \quad \varepsilon^{1 + \ln(T/t)} \ge 2 \varepsilon^{K+1}.
        \end{equation*}
        Since \(\varepsilon\le 1/2\), we have \(2 \varepsilon^{K+1} \le \varepsilon^K\), and the above is implied by \(\varepsilon^{1 + \ln(T/t)} \ge \varepsilon^{K}\).
        Finally \(\ln(T/t) \le K - 1 \Longleftrightarrow T/e^{K-1} \le t\) ensures that \(\opt_{S_t^{(K)}}(\alpha) \ge \minwidth\).  Thus we have 
        \begin{equation}
            \volmax \le \mumax \opt_{S_t^{(K)}}(\alpha), \qquad \volavg \le \muavg \opt_{S_t^{(K)}}(\alpha). \label{eq:get-rid-of-minwidth}
        \end{equation}
    
        Now we analyze the performance of \(\alg\) on \(S_t^{(K)}\).  Since \(S_t^{(K)}\) is identical to \(S^{(K)}\) for the first \(t\) days, \(\alg\) must have the same expected performance over \(W_t\) on \(S_t^{(K)}\) as it does on \(S^{(K)}\).  Thus by our earlier bound, we have that the expected average volume of \(\alg\) over \(W_t\) of \(S_t^{(K)}\) is \(\ge \varepsilon^{\ln(T/t)} - \varepsilon^{K+1}.\)  We have that 
        \begin{equation*}
            \volmax \ge \varepsilon^{\ln(T/t)} - \varepsilon^{K+1}, \qquad \volavg \ge \frac{w}{T} (\varepsilon^{\ln(T/t)} - \varepsilon^{K+1}).
        \end{equation*}
        Combining with (\ref{eq:get-rid-of-minwidth}) and our bound on \(\opt\) (\ref{eq:opt-value}) we get 
        \begin{equation}
            \mumax \ge \frac{\varepsilon^{\ln(T/t)} - \varepsilon^{K+1}}{\varepsilon^{1 + \ln(T/t)} - \varepsilon^{K+1}} \ge \frac{1}{\varepsilon}, \qquad \muavg \ge \frac{w}{T} \cdot \frac{\varepsilon^{\ln(T/t)} - \varepsilon^{K+1}}{\varepsilon^{1 + \ln(T/t)} - \varepsilon^{K+1}} \ge \frac{w}{T} \cdot \frac{1}{\varepsilon}. \label{eq:iid-mu-bounds}
        \end{equation}

        \item \(\alg\) misses every window \(W_t\) for \(T/e^{K-1} \le t \le T/e\).  We lower bound the expected number of mistakes \(\alg\) makes on \(S^{(K)}\).
        \begin{align}
            M &= \sum_{t = 1}^T M_t 
            ~\ge \sum_{t = w}^T \frac{1}{w} \sum_{j = 0}^w M_{t - j} \nonumber \\
            &\ge \sum_{t = \max\{w, T/e^{K-1}\}}^{T/e} \frac{1}{e} \frac{\alpha T}{t} \nonumber\\
            &=\frac{\alpha T}{e} \left(H_{T/e} - H_{\max\{w, T/e^{K-1}\} - 1} \right) &\text{$H_n$ the $n$th harmonic number}\nonumber\\
            &\ge \frac{\alpha T}{e} \left( \ln(T/e) - (\ln(\max\{w, T/e^{K-1}) + 1) \right) \nonumber\\
            &= \frac{1}{e} \left( \min\{\ln (T/ew), ~K-2\}) - 1 \right) \alpha T. \label{eq:iid-mistakes}
        \end{align}
    \end{enumerate}

    We set \(w/T = e^{-K}\).  Then, assuming \(\minwidth \le \varepsilon^{K+1}\), since \(K \le \ln(1/\alpha)\), we have that \(\alg\) must either have 
    \begin{equation*}
        \mumax \ge \frac{1}{\varepsilon} \quad\text{and } \muavg \ge \frac{\alpha}{\varepsilon}
    \end{equation*}
    from (\ref{eq:iid-mu-bounds}), or it must make 
    \begin{equation*}
        \ge \frac{1}{e} ( K - 3 ) \alpha T
    \end{equation*}
    mistakes in expectation on the i.i.d.\ sequence \(S^{(K)}\), from (\ref{eq:iid-mistakes}).  

    To get a bound based on \(\mumax\), we set \(\varepsilon = \frac{1}{\mumax}\), which is possible as long as \(\mumax \ge 2\). We set \(K\) such that \(\minwidth = \varepsilon^{K+1}\), which is possible for \(\minwidth \ge (1/\mumax)^{\ln(1/\alpha)+1}\).  
    This gives us that \(K = \Omega (\ln(1/\minwidth)/\ln\mumax)\).
    As \(\minwidth\) gets smaller, the problem only becomes more general, so we have that for any \(\minwidth > 0\) and \(\mumax \ge 2\), either \(\alg\) has maximum multiplicative approximation \(\ge \mumax\) or \(\alg\) makes 
    \[\Omega \left( \min \left\{ \ln(1/\alpha), ~\frac{\ln (1/\minwidth)}{\ln \mumax} \right\} \alpha T \right)\]
    mistakes in expectation on some i.i.d.\ sequence.  

    To get a bound based on \(\muavg\), we set \(\varepsilon = \frac{\alpha}{\muavg}\), which is possible as long as \(\muavg \ge 2 \alpha\).  We set \(K\) such that \(\minwidth = \varepsilon^{K+1}\), which is possible for \(\minwidth \ge (\alpha/\muavg)^{\ln(1/\alpha) + 1}\).  This gives us that \(K = \Omega ( \ln (1/\minwidth) / (\ln \muavg + \ln (1/\alpha)) )\).  As \(\minwidth\) gets smaller, the problem only becomes more general, so we have that for any \(\minwidth > 0 \) and \(\muavg \ge 2 \alpha\), either \(\alg\) has average multiplicative approximation \(\ge \muavg\) or \(\alg\) makes 
    \[\Omega \left( \min \left\{ \ln(1/\alpha), ~\frac{\ln (1/\minwidth)}{\ln \mumax + \ln(1/\alpha)} \right\} \alpha T \right)\]
    mistakes in expectation on some i.i.d.\ sequence.  

    Finally, we note that while this argument considered the performance of \(\alg\) on intervals with one end at \(\varepsilon^{K+1}\) for ease of notation, we could have instead considered a symmetric version of the sequences, where we map each day \(t\)s input that is drawn \(y_t \sim \distribution_t\) to 
    \[y_t' = \frac{1}{2} + \frac{\unif\{\pm 1\} }{2} \left(y_t - \varepsilon^{K-1}\right).\]
    see for example \Cref{fig:symmetric-iid-lower-bound}.
    This leads to the same lower bounds on volume. This also ensures that the input distribution on each day is unimodal and symmetric around median \(1/2\), because the distributions \(\distribution_t\) in this construction are not supported below \(\varepsilon^{K+1}\), and have decreasing density with \(y\) increasing from \(\varepsilon^{K+1}\).  Thus the same argument gives a lower bound against algorithms that are competitive on such sequences.
\end{proof}

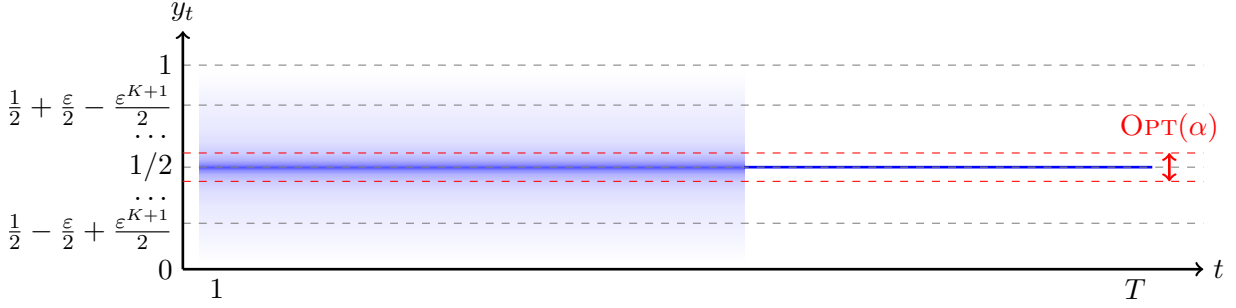
\begin{figure}[!t]
    \centering
    \begin{tikzpicture}[scale=0.9]
        \shade[bottom color = blue!66.66!white, top color = blue!44.44!white] (0.25, 3/2) rectangle (8.25, 3/2 + 3/2*0.0877 - 3/2*0.0585);
        \shade[bottom color = blue!44.44!white, top color=blue!29.62!white] (0.25, 3/2 + 3/2*0.0877 - 3/2*0.0585) rectangle (8.25, 3/2 + 3/2*0.1316- 3/2*0.0585); 
        \shade[bottom color = blue!29.62!white, top color=blue!19.75!white] (0.25, 3/2 + 3/2*0.1316 - 3/2*0.0585) rectangle (8.25, 3/2 + 3/2*0.1975- 3/2*0.0585);
        \shade[bottom color = blue!19.75!white, top color = blue!13.16!white] (0.25, 3/2 + 3/2*0.1975- 3/2*0.0585) rectangle (8.25, 3/2 + 3/2*0.2962 - 3/2*0.0585);
        \shade[bottom color = blue!13.16!white, top color=blue!8.77!white] (0.25, 3/2 + 3/2*0.2962 - 3/2*0.0585) rectangle (8.25, 3/2 + 3/2*0.4444 - 3/2*0.0585);
        \shade[bottom color = blue!8.77!white, top color = blue!5.85!white] (0.25, 3/2 + 3/2*0.4444 - 3/2*0.0585) rectangle (8.25, 3/2 + 3/2*0.6666- 3/2*0.0585);
        \shade[bottom color = blue!5.85!white, top color = white] (0.25, 3/2 + 3/2*0.6666 - 3/2*0.0585) rectangle (8.25, 3/2 + 3/2*1 - 3/2*0.0585);

        \shade[top color = blue!66.66!white, bottom color = blue!44.44!white] (0.25, 3/2) rectangle (8.25, 3/2 - 3/2*0.0877 + 3/2*0.0585);
        \shade[top color = blue!44.44!white, bottom color=blue!29.62!white] (0.25, 3/2 - 3/2*0.0877 + 3/2*0.0585) rectangle (8.25, 3/2 - 3/2*0.1316 + 3/2*0.0585); 
        \shade[top color = blue!29.62!white, bottom color=blue!19.75!white] (0.25, 3/2 - 3/2*0.1316 + 3/2*0.0585) rectangle (8.25, 3/2 - 3/2*0.1975+ 3/2*0.0585);
        \shade[top color = blue!19.75!white, bottom color = blue!13.16!white] (0.25, 3/2 - 3/2*0.1975 + 3/2*0.0585) rectangle (8.25, 3/2 - 3/2*0.2962 + 3/2*0.0585);
        \shade[top color = blue!13.16!white, bottom color=blue!8.77!white] (0.25, 3/2 - 3/2*0.2962 + 3/2*0.0585) rectangle (8.25, 3/2 - 3/2*0.4444 + 3/2*0.0585);
        \shade[top color = blue!8.77!white, bottom color = blue!5.85!white] (0.25, 3/2 - 3/2*0.4444 + 3/2*0.0585) rectangle (8.25, 3/2 - 3/2*0.6666 + 3/2*0.0585);
        \shade[top color = blue!5.85!white, bottom color = white] (0.25, 3/2 - 3/2*0.6666 + 3/2*0.0585) rectangle (8.25, 3/2 - 3/2*1 + 3/2*0.0585);

        \draw[line width = 1pt, blue] (8.25, 3/2) -- (14.25, 3/2);

        \draw[->, line width=1pt] (0, 0) -- (15, 0);
        \draw[dashed, gray] (0, 3) -- (15, 3);
        \draw[dashed, gray] (0, 3/2) -- (15, 3/2);
        \draw[dashed, gray] (0, 3/2 + 3/2*0.6666 - 3/2*0.0585) -- (15, 3/2 + 3/2*0.6666 - 3/2*0.0585);
        \draw[dashed, gray] (0, 3/2 -3/2*0.6666 + 3*0.0585) -- (15, 3/2 - 3/2*0.6666 + 3*0.0585);
        \draw[->, line width=1pt] (0, 0) -- (0, 3.5);
        \node[anchor=north] (t=1) at (0.5, 0) {$1$};
        \node[anchor=north] (t=T) at (14, 0) {$T$};
        \node[anchor=west] (time) at (15, 0) {$t$};
        \node[anchor=east] (y=0) at (0, 0) {$0$};
        \node[anchor=east] (y=1) at (0, 3) {$1$};
        \node[anchor=east] (y=1) at (0, 3/2) {$1/2$};
        \node[anchor=east] (yopt) at (0, 3/2 + 3/2 *0.6666 - 3/2*0.0585) {$\frac{1}{2} + \frac{\varepsilon}{2} - \frac{\varepsilon^{K + 1}}{2}$};
        \node[anchor=east] (yopt) at (0, 3/2 - 3/2 *0.6666 + 3/2*0.0585) {$\frac{1}{2} - \frac{\varepsilon}{2} + \frac{\varepsilon^{K + 1}}{2}$};
        \node[anchor=east] (yopt) at (0, 3/2 + 3/4 *0.6666 - 3/4*0.0585) {$\dots$};
        \node[anchor=east] (yopt) at (0, 3/2 - 3/4 *0.6666 + 3/4*0.0585) {$\dots$};
        \node[anchor=south] (y) at (0, 3.5) {$y_t$};
        \draw[dashed, red] (0, 3/2 + 3/2*0.1975 - 3/2*0.0585) -- (15, 3/2 + 3/2*0.1975 - 3/2*0.0585);
        \draw[dashed, red] (0, 3/2 -3/2*0.1975 + 3/2*0.0585) -- (15, 3/2 - 3/2*0.1975 + 3/2*0.0585);
        \draw[<->, line width=1pt, red] (14.5, 3/2 - 3/2*0.1975 + 3/2*0.0585) -- (14.5, 3/2 + 3/2*0.1975 - 3/2*0.0585);
        \node[anchor=south, red] (c) at (14.5, 3/2 + 3/2*0.1975 - 3/2*0.0585) {$\textsc{Opt}(\alpha)$};
    \end{tikzpicture}
    \caption{Symmetric version of i.i.d.\ lower bound construction in \Cref{fig:iid-lower-bound}}
    \label{fig:symmetric-iid-lower-bound}
\end{figure}

\newcommand{\bfS}{\mathbf{S}}

\section{Uniform Convergence for Exchangeable Sequences}
\label{sec:uniform-convergence}

In this section, we show that exchangeability suffices for standard uniform convergence bounds based on VC dimension.  While this is already known in the literature, we provide a proof for completeness.  We will consider a domain $\mathcal{X}$, and a hypothesis class $\mathcal{H}$ comprised of functions of the form $h: \mathcal{X} \to \{0,1\}$. Given a set of sample points $\{S_1, \dots, S_T\} \in \mathcal{X}$ and a hypothesis $h \in \mathcal{H}$, we define the loss $\ell(h, \{S_1, \dots, S_T\})$ of hypothesis $h$ on the samples $\{S_1, \dots, S_T\}$ to be the fraction of points in $\{S_1, \dots, S_T \}$ that are mapped to $1$ i.e., 
$$\ell\big(h, \{ S_1, \dots, S_T\}\big) = \frac{1}{T} \sum_{i=1}^T h(S_i).$$
For an exchangeable sequence of samples $\bfS=(S_1, \dots, S_T)$, we will overload notation and define $\ell(h,\bfS) = \ell(h, \{S_1, \dots, S_T\})$. 

\begin{lemma}[Uniform Convergence for Exchangeable Sequences]\label{lem:exchangeable-uniform-convergence}
    \addcontentsline{toc}{subsection}{\emph{\Cref{lem:exchangeable-uniform-convergence}:} Uniform Convergence for Exchangeable Sequences}
    Let \(\mathcal{H}\) be a hypothesis class of functions from a domain \(\mathcal{X}\) to \(\{0, 1\}\) of VC-dimension \(d\). 
    Let \(\mathbf{S} = (S_1, \dots, S_T)\) be an exchangeable sequence of random variables over elements of \(\mathcal{X}\).  
    Then there exists a universal constant \(C\) such that for any fixed setting of \(t \le T, \delta>0, \varepsilon>0\) such that 
    \[t \ge C \Big(\frac{d \log(d/\varepsilon) + \log(1/\delta)}{\varepsilon^2}\Big), \]
    we have that with probability \(\ge 1 - \delta\) over the exchangeability of \(\mathbf{S}\), 
    for all \(h \in \mathcal{H}\) simultaneously, 
    \[ \ell(h, \mathbf{S}) - \varepsilon ~\le~ \ell(h, \{S_1, \dots, S_t \}) ~\le~ \ell(h, \mathbf{S}) + \varepsilon.\]
\end{lemma}

We follow the standard symmetrization argument for uniform convergence bounds, with some extra care to make them work with exchangeability.  
We use the following Hoeffding concentration bound for the sum of exchangeable random variables. 
\begin{lemma}\label{lem:hoeffdinginequality}\citep*{Hoeffding01031963}
Let $X_1, \dots, X_n$ be an exchangeable sequence of random variables with $X_i \in [0,1]$ and mean $\E[X_1]=\mu$. Then we have the following upper tail bound
\[ \Pr\Big[\frac{1}{n} \sum_{i=1}^n X_i  - \mu  \ge \lambda \Big] \le  \exp\Big( - n \lambda^2/2  \Big), \]
and we get an identical bound the lower tail.
\end{lemma}
The above lemma is a consequence of Hoeffding's concentration bounds for sampling with replacement~\citep*{Hoeffding01031963,Serfling1974ProbabilityIF}. This also uses the fact that any finite sequence of exchangeable variables being expressible as a mixture of {\em urn} sequences (i.e., sampling with replacement) ~\citep*{Diaconis1980FiniteES}.    
(See also \citep*{Barber2024HoeffdingAB} for concentration bounds for general weighted sums of exchangeable r.v.s.)

We now proceed to the proof of Lemma~\ref{lem:exchangeable-uniform-convergence}.

\begin{proof}

First we note from the projection property of exchangeable sequences, the subsequence $\bfS=(S_1, \dots S_t)$ is also exchangeable. The proof of uniform convergence follows the standard symmetrization approach as given by \citet*{Bousquet2004}. 

Let $\bfS'=(S'_1, \dots, S'_t)$ be an independent sample from the same exchangeable distribution as $(S_1, \dots, S_t)$. Recall that for a hypothesis $h \in \mathcal{H}$, $\ell(h,\mathbf{S}) = \E_{\mathbf{S}}[\ell(h,(S_1, \dots, S_T))]$ is the expected loss, while $\ell(h,\bfS)$ (and $\ell(h,\bfS')$) is the empirical loss on the samples $S_1, \dots, S_t$ (respectively $S'_1, \dots, S'_t$). 

First from Hoeffding's bound for exchangeable sequences (Lemma~\ref{lem:hoeffdinginequality}), we have 
\begin{equation}\label{eq:conc:exchangeable}
    \Pr\Big[ \ell(h, \mathbf{S}) - \ell(h,\bfS_t)  \ge \lambda \Big] \le \exp\Big( -t \lambda^2/2\Big),  \text{ and }  \Pr\Big[  \ell(h,\bfS'_t) - \ell(h, \mathbf{S})  \ge \lambda \Big] \le \exp\Big( -t \lambda^2/2\Big).
\end{equation}

We first prove the following claim through symmetrization. The claim allows us to relate the error on one sample to the discrepancy between two independent samples from the distribution.    

\begin{claim}\label{claim:symmetrization}
For any $\lambda>0$, such that $t \lambda^2 \ge 2$, 
\begin{equation}\label{eq:symmetrization}
    \Pr\Big[ \sup_{h \in \mathcal{H}}  \ell(h, \mathbf{S}) - \ell(h,\bfS_t)  > \lambda \Big] \le 2  \Pr\Big[ \sup_{h \in \mathcal{H}}  \ell(h, \bfS'_t) - \ell(h,\bfS_t)  > \lambda/2 \Big]
\end{equation}
\end{claim}
\begin{proof}
Let $h^\star \in \mathcal{H}$ be a hypothesis that achieves the supremum on the left side of \eqref{eq:symmetrization}. We will lower bound the probability of the event on the right by showing that it is satisfied for $h^\star$. For any $h \in \mathcal{H}$, 
\begin{align*}
\mathbf{1}\Big[ \ell(h,\bfS'_t)- \ell(h,\bfS_t) >  \lambda/2 \Big] &\ge \mathbf{1}\Big[ \big(\ell(h,\bfS)- \ell(h,\bfS_t) > \lambda\big) ~\wedge~ \big(\ell(h,\bfS)- \ell(h,\bfS'_t) < \lambda/2\big) \Big]. 
 \end{align*}
Note that $\bfS'_t$  is independent of $\bfS_t$. Taking expectation on both sides w.r.t. $\bfS'_t$, 
 \begin{align}
 \Pr_{\bfS'_t}\Big[ \ell(h,\bfS'_t)- \ell(h,\bfS_t) >  \lambda/2 \Big] & \ge \mathbf{1}\big[ \ell(h,\bfS_t)- \ell(h,\bfS) > \lambda\big] \cdot  \Pr_{\bfS'_t}\Big[ \ell(h,\bfS'_t)- \ell(h,\bfS) > \lambda/2\Big]. \label{eq:symm:1}
\end{align}
The above inequality holds for any fixed hypothesis $h \in \mathcal{H}$, including any $h^\star$ that achieves the supremum of the left side of \eqref{eq:symmetrization}. 
Note that $\E_{\bfS'_t}[\ell(h,\bfS'_t)]=\ell(h,\bfS)$. The event on the right does not depend on $\bfS_t$, and follows by applying the Chebyshev inequality on the exchangeable sequence $\bfS'_t=(S'_1, \dots, S'_t)$.  Hence for any   fixed hypothesis $h$ (including $h^\star$) 
$$\Pr_{\bfS'_t}\Big[ \ell(h,\bfS'_t)- \ell(h,\bfS) > \lambda/2\Big] \le \frac{\text{Var}[\ell(h,\bfS'_t)]}{\lambda^2} \le \frac{1}{4 t\lambda^2} \le \frac{1}{2}. $$
Note that we have a variance upper bound for $\ell(h,\bfS'_t)$ is due to exchangeability (sampling with replacement). Substituting in \eqref{eq:symm:1} and taking expectation w.r.t. $\bfS_t$, we have
\begin{align*}
\Pr\Big[ \sup_{h \in \mathcal{H}} \ell(h,\bfS'_t)- \ell(h,\bfS_t) >  \lambda/2 \Big] &\ge  \Pr_{\bfS_t}\big[ \sup_{h \in \mathcal{H}} \ell(h,\bfS_t)- \ell(h,\bfS) > \lambda\big] \times  \frac{1}{2}.
 \end{align*}
By rearranging terms, we get the claim.
\end{proof}

We now upper bound the right side of \eqref{eq:symmetrization}. For a fixed $h \in \mathcal{H}$, we have from \eqref{eq:conc:exchangeable} 
\begin{align*}
\Pr_{\bfS_t, \bfS'_t}\Big[ \ell(h, \bfS'_t) - \ell(h,\bfS_t)  \ge \frac{\lambda}{2} \Big] & \le \Pr\Big[ \ell(h, \bfS'_t) -\ell(h,\bfS)  \ge \frac{\lambda}{4} \Big] + \Pr\Big[ \ell(h, \bfS) -\ell(h,\bfS_t)  \ge \frac{\lambda}{4} \Big] 
\le 2 \exp\Big( - \frac{t\lambda^2}{32} \Big).
\end{align*}

Now we observe that for a given hypothesis $h \in \mathcal{H}$, $\ell(h, \bfS'_t) - \ell(h,\bfS_t)$ is an empirical sum and only depends on the $2n$ samples $S_1, \dots, S_t, S'_1, \dots, S'_t$. Hence the number of distinct hypothesis that we need to union bound over is at most the shattering number $N(2t,d) = \sum_{j=1}^{d} \binom{2t}{d}$. Hence
\begin{align*}
\Pr\Big[ \sup_{h \in \mathcal{H}}  \ell(h, \mathbf{S}) - \ell(h,\bfS_t)  > \lambda \Big] &\le 2 \Pr_{\bfS_t, \bfS'_t}\Big[ \sup_{h \in \mathcal{H}}\ell(h, \bfS'_t) - \ell(h,\bfS_t)  \ge \frac{\lambda}{2} \Big] \\
& \le 4 N(2t,d) \cdot \exp\Big( - \frac{t\lambda^2}{32} \Big) \le \delta,
\end{align*}
for our choice of $t$ and $\lambda$. This finishes the proof. 
\end{proof}

\section{Future Directions}
In this work we provide a new framework to measure the efficiency of online conformal prediction algorithms.  We apply this to studying the setting where the response variables are in the unit interval \([0, 1]\), the notion of efficiency is minimizing the volume (Lebesgue measure) of the prediction set, and the family of sets that we compete against is subintervals \([a, b] \subseteq [0, 1]\).  We study this in the unsupervised (featureless) setting for arbitrary and exchangeable input sequences.  We view this as the most basic and fundamental setting in which to explore efficiency optimality.  We believe that there are many opportunities for future work to explore efficiency optimality for online conformal prediction in a broad range of other settings that are widely studied in conformal prediction and present new challenges.  We provide a few examples of these settings.

\begin{itemize}
    \item \emph{Other prediction sets.} Even in the setting where the target variables live in the unit interval \([0, 1]\), there are other families of prediction sets that we can consider competing against.  For example, in the standard setting, \cite{gao2025volumeoptimalityconformalprediction} provides an algorithm that competes with the volume-optimal \emph{union of \(k\) intervals}.  Can we design an algorithm that competes with unions of intervals in the online setting?  
    
    \item \emph{Other geometries.}  It would be interesting to extend this result to other spaces of target variables \(\mathcal{Y}\).  For example, what if instead of being points in the unit interval, our target variables were points in high-dimensional space?  Are there natural classes of prediction sets that we can compete against in the online setting?

    \item \emph{Other notions of efficiency.}  In this work we consider minimizing the volume (Lebesgue measure) of the prediction sets played.  Depending on the task that the conformal predictor is used for, there could be other natural notions of efficiency.  For example, we could consider minimizing other measures of the prediction set.  Alternatively, the cost of a prediction set could correspond to some downstream decision-making cost.  Can we design online conformal prediction algorithms that minimize other natural notions of cost on the sets played?

    \item \emph{Other forms of distribution shift.}  In this work, we consider exchangeable sequences, which exhibit no distribution shift, and arbitrary sequences that exhibit arbitrary distribution shifts.  One conclusion of this work is that, even in our basic setting, these two types of sequences admit very different guarantees.  Are there natural but more benign forms of distribution shift for which we can design algorithms that have stronger guarantees than what is possible for fully arbitrary sequences?

    \item \emph{Supervised settings and conditional guarantees.}  In the supervised setting of conformal prediction, each target variable \(y_t \in \mathcal{Y}\) is associated by a feature variable \(x_t \in \mathcal{C}\) which may be predictive of it.  On each day, the algorithm observes \(x_t\) and then produces a prediction set \(C(x_t)\) to try to capture \(y_t\).  It would be interesting to explore efficiency optimality in this supervised setting.  This could involve interesting new notions of optimality.  Now that we have access to the \(x_t\), it does not necessarily make sense to compete with the best \emph{fixed} prediction set in hindsight, as that eliminates strategies that vary as a function of \(x_t\).  It would be interesting to use the features to get guarantees against stronger baselines that are allowed to depend on the features.  
\end{itemize}

\addcontentsline{toc}{section}{Acknowledgements}
\section*{Acknowledgements}
I thank my advisor Aravindan Vijayaraghavan for helpful discussions, and for the proof that uniform convergence extends to exchangeable sequences (\Cref{sec:uniform-convergence}).  This work was supported by the NSF-funded Institute for Data, Econometrics, Algorithms and Learning (IDEAL) through the grant NSF ECCS-2216970, the NSF via grant CCF-2154100, and the Northwestern Presidential Fellowship.

\addcontentsline{toc}{section}{References}
\bibliographystyle{plainnat}
\bibliography{ref}

\end{document}